\documentclass[final,12pt]{arxiv2025}

\usepackage{algorithm}
\usepackage{algorithmic}
\usepackage{url}
\usepackage{lastpage}
\usepackage{hyperref}
\newtheorem{assumption}{Assumption}

\usepackage{enumitem,amsfonts,amssymb,mathtools,bm,pifont,bbm,listings,xcolor,threeparttable,ulem,comment}

\usepackage{natbib}
\usepackage{multirow}

\def\Rn{\mathbb{R}}  \def\zero{\mathbf{0}}
\def\fro {\mathsf{F}}

\definecolor{mydarkgreen}{RGB}{39,130,67}
\definecolor{mydarkred}{RGB}{192,25,25}

\usepackage{relsize} 

\newcommand{\beq}{\begin{eqnarray}}
\newcommand{\eeq}{\end{eqnarray}}
\newcommand{\beqq}{\begin{equation}}
\newcommand{\eeqq}{\end{equation}}
\newcommand{\bel}{\begin{align}}
\newcommand{\eel}{\end{align}}
\newcommand{\la}{\langle}
\newcommand{\ra}{\rangle}
\newcommand{\noi}{\noindent}
\newcommand{\nn}{\nonumber}
\def\noi{\noindent}
\def\nn{\nonumber}
\def\la{\langle}
\def\ra{\rangle}

\def\Diag{{\rm{Diag}}}

\newcommand{\step}[1]{\text{\ding{\numexpr#1+171\relax}}}

\def\Deltas{\boldsymbol{\Delta}}

\def\epsilons{\boldsymbol{\epsilon}}

\def\ts{\textstyle}

\def\a{\mathbf{a}}\def\b{\mathbf{b}}\def\g{\mathbf{g}}\def\m{\mathbf{m}}\def\q{\mathbf{q}}\def\s{\mathbf{s}}\def\v{\mathbf{v}}\def\x{\mathbf{x}}\def\y{\mathbf{y}}\def\z{\mathbf{z}}\def\G{\mathbf{G}}\def\P{\mathbf{P}}\def\X{\mathbf{X}}\def\v{\mathbf{v}}

\def\FF{\mathcal{F}}\def\OO{\mathcal{O}}\def\ZZ{\mathcal{Z}}

\def\EEE{\mathbb{E}}\def\GGG{\mathbb{G}}\def\MMM{\mathbb{M}}\def\SSS{\mathbb{S}}

\def\NO{\ding{56}}
\def\YES{\ding{52}}
\newcounter{optalg}

\newcommand{\xixi}[2]{\if\relax\detokenize\expandafter{\romannumeral-`\q#1}\relax#2\else \overset{\step{#1}}{#2}\fi}

\newcommand\hG{{\hat{G}}}

\usepackage{times}

\title[OptEMA]{OptEMA: Adaptive Exponential Moving Average for Stochastic Optimization with Zero-Noise Optimality}

\vspace{10pt}
\coltauthor{
 \Name{Ganzhao Yuan} \Email{yuanganzhao@foxmail.com}\\
 \addr Shenzhen University of Advanced Technology (SUAT), China
}
\begin{document}

\maketitle

\begin{abstract}

Exponential moving averages (EMAs) are a central component of widely used adaptive optimizers such as Adam. However, existing analyses of Adam-style methods often yield suboptimal guarantees in the zero-noise regime, rely on open-loop parameter schedules, or require prior knowledge of smoothness constants. Motivated by these limitations, we introduce OptEMA and analyze two complementary variants: OptEMA-M, which applies an adaptive, decreasing EMA coefficient to the first moment with a fixed second-moment decay, and OptEMA-V, which swaps these roles. At the heart of these variants is a Corrected AdaGrad-Norm coefficient schedule. This formulation renders OptEMA algorithmically closed-loop and Lipschitz-free, meaning its effective stepsizes are trajectory-dependent and require no parameterization via the Lipschitz constant. Under lower-boundedness, unbiasedness, bounded variance, average smoothness, and a bounded stochastic-gradient condition used to control the adaptive normalizers, we prove that both variants achieve the unified noise-adaptive rate $\tilde{\mathcal{O}} \left(T^{-1/2}+\sigma^{1/2}T^{-1/4}\right)$ for the averaged gradient norm. In the zero-noise regime, these bounds automatically reduce to the nearly optimal deterministic rate $\widetilde{\mathcal{O}}(T^{-1/2})$ without manual hyperparameter retuning.



\end{abstract}


\section{Introduction}

\normalem
We study mini-batch stochastic optimization for the objective
\begin{equation}\label{eq:main}
\min_{\x\in\mathbb{R}^{n}} f(\x):=\mathbb{E}_{\xi\sim\mathcal{D}}[f(\x;\xi)] ,
\end{equation}
where $f(\cdot)$ is differentiable and possibly nonconvex, and $\xi$ denotes an i.i.d. sample from an unknown distribution $\mathcal{D}$. Our goal is to find an $\epsilon$-approximate stationary point $\x$ such that $\mathbb{E}[\|\nabla f(\x)\|]\le \epsilon$. Throughout the analysis, we assume lower-boundedness, unbiased stochastic gradients with bounded variance, and average smoothness of \(f\). We additionally impose a bounded stochastic-gradient condition to control the adaptive EMA normalizers and pathwise tracking terms. 

\paragraph{The Dominance of EMA.} Adaptive gradient methods, such as RMSProp \cite{hinton2012neural} and Adam \cite{KingmaBaCILR15}, have become standard tools for training deep neural networks, often showing strong empirical performance compared to vanilla Stochastic Gradient Descent (SGD). A key ingredient behind their practical success is the Exponential Moving Average (EMA) mechanism, which provides a memory-efficient way to accumulate historical gradient information. In particular, EMA is used to form moving estimates of the first moment and the second raw moment of stochastic gradients, thereby inducing momentum-like averaging and coordinate-wise rescaling. Intuitively, EMA can be viewed as a temporal low-pass filter on the gradient sequence \cite{qian1999momentum}: it suppresses rapid stochastic fluctuations while retaining more persistent directional information. By smoothing updates and damping oscillations, EMA often yields more stable optimization dynamics, which is particularly valuable in the highly nonconvex loss landscapes encountered in modern deep learning.

\paragraph{Zero-Noise Optimality.} In this paper, zero-noise optimality means that the stochastic convergence guarantee automatically recovers the deterministic first-order rate when the variance level vanishes. More precisely, for a randomly selected output \(x_{\rm out}\), a noise-adaptive optimizer is said to be zero-noise optimal if its bound reduces to $\mathbb{E}\bigl[\|\nabla f(x_{\rm out})\|\bigr] \le \widetilde{\mathcal O}(T^{-1/2})$ when $\sigma=0$. This property is stronger than merely proving convergence under stochastic gradients, because the bound must improve continuously as the noise level decreases and must not retain a residual stochastic term in the deterministic limit.

\paragraph{Limitations of Existing EMA Analyses.} Despite their empirical success, the theoretical understanding of EMA-based optimizers remains incomplete. In particular, several Adam-type analyses do not automatically recover the deterministic rate $\mathcal{O}(T^{-1/2})$ when the stochastic variance vanishes ($\sigma=0$). Instead, available guarantees for Adam-type methods often yield the slower rate $\mathcal{O}(T^{-1/4})$ even in deterministic or zero-noise regimes \cite{chen2019convergence, chen2022towards,huang2021super,zhang2022adam,liang2025convergence}. This indicates that existing bounds may not fully capture the noise-adaptive behavior expected from an optimizer in the zero-noise or lower-noise limit. Moreover, many EMA-based adaptive methods use fixed or pre-scheduled EMA coefficients and global learning-rate schedules. These open-loop choices are not directly coupled to the observed optimization trajectory, and therefore may require careful tuning when the smoothness or noise level is unknown. These limitations motivate a closed-loop EMA mechanism whose coefficients and effective stepsizes are determined by trajectory-dependent feedback.

\paragraph{Motivation and Contributions.} These theoretical limitations and practical tuning burdens motivate the search for an EMA-based optimizer that is both trajectory-adaptive and theoretically noise-adaptive. To this end, we propose OptEMA (adaptive exponential moving average with zero-noise optimality). The main novelty is structural rather than purely rate-based: compared with AdaGrad-Norm, our goal is not to improve the worst-case rate under the same bounded-gradient assumption, but to embed a corrected AdaGrad-Norm feedback mechanism into the standard Adam-style EMA architecture. Our contributions are summarized as follows.

\begin{itemize}

\item \textbf{Closed-loop EMA design.}
We redesign the standard EMA mechanism as a closed-loop feedback rule. This leads to two complementary variants: OptEMA-M adapts the first-moment EMA coefficient while keeping the second-moment coefficient fixed, whereas OptEMA-V adapts the second-moment EMA coefficient while keeping the first-moment coefficient fixed. In both cases, the EMA coefficients are coupled to the observed optimization trajectory rather than prescribed by an open-loop schedule.

\item \textbf{Corrected AdaGrad-Norm coefficient schedule.}
We introduce a data-dependent coefficient schedule that preserves the monotone and dimension-free nature of AdaGrad-Norm \cite{ward2020adagrad} while correcting its overly aggressive decay through a historical-gradient averaging factor. This correction reduces sensitivity to occasional gradient spikes and provides the key mechanism behind the zero-noise behavior of OptEMA.

\item \textbf{Noise-adaptive convergence guarantees.}
Under lower-boundedness, unbiasedness, bounded variance, average smoothness, and bounded stochastic gradients, we prove that both OptEMA-M and OptEMA-V achieve $\widetilde{\mathcal O}\bigl(T^{-1/2}+\sigma^{1/2}T^{-1/4}\bigr)$ for the averaged gradient norm. Consequently, when \(\sigma=0\), the same bounds reduce to the nearly deterministic rate \(\widetilde{\mathcal O}(T^{-1/2})\).
\end{itemize}

\paragraph{Organization.} The remainder of this paper is organized as follows. Section \ref{sect:related} reviews the related literature, and Section \ref{sect:prel} introduces the assumptions and technical preliminaries. Section \ref{sect:proposed} presents the OptEMA framework and its variants, while Section \ref{sect:rate} provides the iteration complexity analysis. Section \ref{sect:conc} concludes the paper. Detailed proofs are deferred to the appendix.

\section{Related Work} \label{sect:related}

This section reviews the literature most relevant to our work, including stochastic gradient methods, adaptive optimization, momentum-based variance reduction, and recent Lipschitz-free and noise-adaptive algorithms.

\paragraph{Stochastic Gradient Methods.} Stochastic first-order methods, including SGD and momentum-based variants, remain the foundation of large-scale learning. For smooth nonconvex objectives, their analyses typically assume unbiased stochastic gradients, bounded variance, lower-bounded objectives, and suitable smoothness conditions. The achievable convergence rate depends critically on the adopted smoothness model. Under average smoothness, the best known lower bound scales as $\Omega(T^{-1/4})$, whereas stronger individual smoothness can support faster rates such as $\Omega(T^{-1/3})$ \cite{arjevani2023lower}. These distinctions are central to understanding the gap between classical stochastic optimization theory and modern deep-learning optimizers.

\paragraph{Adaptive Gradient Methods.} Adaptive gradient methods were developed to better handle stochastic gradients with heterogeneous coordinate scales, especially in sparse-data regimes. AdaGrad \cite{duchi2011adaptive}  introduced data-dependent per-coordinate stepsizes through accumulated gradient statistics, which naturally amplify updates on infrequent coordinates and damp those on frequently active ones. RMSProp \cite{hinton2012neural} replaced monotone accumulation with an EMA of squared gradients to avoid overly aggressive decay. Adam \cite{KingmaBaCILR15} then combined first-moment momentum with second-moment EMA scaling, and has since become the dominant adaptive optimizer in practice. Later variants refined this design in different ways: AdamW \cite{loshchilovdecoupled} decoupled weight decay from adaptive updates, while AMSGrad \cite{reddi2018convergence} introduced a running maximum of second moments to address known non-convergence phenomena in vanilla Adam.

\paragraph{Momentum-based Variance Reduction.} STORM \cite{cutkosky2019momentum} was among the first methods to combine recursive gradient estimation with a momentum-style update in stochastic optimization. STORM$^+$ \cite{levy2021storm} further demonstrated that, with appropriately designed adaptive stepsize strategies, such estimators can achieve both Lipschitz-free and noise-adaptive guarantees. META-STORM \cite{liu2022meta} relaxed the bounded-objective assumption, while AdaSTORM \cite{JiangAdaStorm} further removed both the bounded-objective and bounded-gradient assumptions, although it relies on a hybrid stepsize scheme that combines a constant step size of order $\mathcal{O}(T^{-1/3})$ with a momentum-adapted one. SuperAdam \cite{huang2021super} offered a broader perspective by integrating momentum, adaptive scaling, and matrix preconditioning. MARS \cite{yuan2025mars} revisited recursive momentum by introducing an additional variance-reduction correction term to improve the quality of the gradient estimator. Together, these works help bridge modern stochastic optimization theory and practical adaptive optimization.

\paragraph{Adam-Type versus STORM-Type Methods.} Adam-type and STORM-type methods represent two complementary directions in stochastic optimization. STORM-type methods can achieve faster asymptotic rates, such as $\mathcal{O}(T^{-1/3})$, but typically rely on individual smoothness, which is stronger than the average smoothness commonly used in analyses of Adam-type methods. They also incur additional computational overhead, since maintaining the recursive variance-reduced estimator requires evaluating both $\nabla f(\x_t,\xi_t)$ and $\nabla f(\x_{t-1},\xi_t)$ on the same mini-batch at each iteration. In large-scale deep learning, this extra cost can substantially reduce their practical appeal. As a result, STORM-type methods provide important theoretical benchmarks, while Adam-type methods remain the dominant practical choice. OptEMA follows this Adam-style design: it preserves the standard one-gradient EMA update rule and operates under average smoothness. The OptEMA analysis gives the unified noise-adaptive rate $\widetilde{\mathcal O}\bigl(T^{-1/2}+\sigma^{1/2}T^{-1/4}\bigr)$ (see Theorems~\ref{the:it:EMA:M} and~\ref{the:it:EMA:V}), compared with the STORM-type dependence $\widetilde{\mathcal O}\bigl(T^{-1/2}+\sigma^{1/3}T^{-1/3}\bigr)$ in \cite{levy2021storm}. This comparison is meaningful because OptEMA is analyzed under average smoothness and uses only one stochastic gradient per iteration, whereas many STORM-type recursive estimators rely on individual smoothness and require two gradient evaluations on the same mini-batch. At the level of the stochastic term, we have $\sigma^{1/2}T^{-1/4} < \sigma^{1/3}T^{-1/3}\Longleftrightarrow
\sigma<T^{-1/2}$. Therefore, up to logarithmic factors, OptEMA enjoys a smaller stochastic term in the lower-noise regime \(0<\sigma<T^{-1/2}\), while retaining the standard one-gradient Adam-style EMA structure.

\paragraph{Lipschitz-Free and Noise-Adaptive Algorithms.} Representative methods in this regime, such as AdaGrad-Norm \cite{ward2020adagrad}, and STORM+ \cite{levy2021storm}, use the optimization trajectory as feedback to adapt the learning rate. These methods are both Lipschitz-free and noise-adaptive: they do not require prior knowledge of the Lipschitz smoothness constant and automatically adjust to the noise level $\sigma$. In particular, when $\sigma=0$, their bounds recover the deterministic convergence rate for smooth nonconvex optimization. However, these approaches remain structurally different from standard EMA-based optimizers used in deep learning. AdaGrad-Norm relies on cumulative averaging rather than exponential moving averages (EMA), while the analyses of STORM$^+$ and META-STORM are tied to variance-reduction architectures. By contrast, OptEMA preserves the standard EMA structure of Adam-type methods; both OptEMA-M and OptEMA-V achieve the unified rate $\widetilde{\mathcal{O}}(T^{-1/2} + \sigma^{1/2}T^{-1/4})$ and therefore retain zero-noise optimality.

\paragraph{Assumptions in Different Algorithms.} OptEMA is analyzed under average smoothness, a lower-bounded objective, unbiased stochastic gradients with bounded variance, and the bounded stochastic gradient condition. This assumption is used to control the adaptive EMA normalizers and the pathwise constants in the current proof. By contrast, several STORM-family analyses \cite{levy2021storm,liu2022meta} additionally impose the stronger individual smoothness condition, bounded objective values (BF), or Hessian-type boundedness (BH), which may be restrictive in large-scale learning. Similarly, under average smoothness, SuperAdam also relies on Hessian-type boundedness assumptions \cite{huang2021super}.

\paragraph{Other Recent Developments.} Recent work has also emphasized hardware efficiency and structural priors. Sign-based methods such as signSGD \cite{bernstein2018signsgd} and Lion \cite{chen2023symbolic} reduce memory and communication costs, while AdaFactor \cite{shazeer2018adafactor} compresses second-moment statistics through factorization. Matrix-structured optimizers, such as Muon \cite{chang2025convergence,MuonICLR26}, use orthogonalized gradient updates to improve conditioning. These efficiency- and structure-oriented approaches, however, are typically paired with open-loop stepsize rules and thus offer limited adaptivity to heterogeneous noise. In a different direction, methods such as D-Adaptation \cite{defazio2023learning} and DoG \cite{ivgi2023dog} use the optimization trajectory itself as feedback to adapt the learning rate. This trajectory-aware perspective is closely related to the closed-loop EMA design of OptEMA.

\paragraph{Summary.} Existing adaptive methods are frequently bottlenecked by weak zero-noise optimality, restrictive boundedness assumptions, rigid open-loop step sizes, and the need for known Lipschitz constants. The proposed OptEMA framework addresses these issues from the closed-loop EMA perspective: it is Lipschitz-free, noise-adaptive, and avoids bounded objective and Hessian-type assumptions, while the present analysis still retains bounded stochastic gradients. A comprehensive comparison of existing stochastic gradient methods is detailed in Table \ref{tab:main}.

\begin{table*}[t]
\renewcommand{\arraystretch}{1.2}
\centering
\caption{Comparison of existing mini-batch first-order stochastic gradient methods for finding an $\epsilon$-approximate solution $\mathbf{x}$ satisfying $\EEE[\|\nabla f(\mathbf{x})\|]\leq \epsilon$. The notation $\widetilde{\mathcal{O}}(\cdot)$ hides polylogarithmic factors, while $\mathcal{O}(\cdot)$ hides constants. All methods assume unbiased gradients with bounded variance, a lower-bounded objective, and smoothness; the column ``Additional Assumptions'' lists extra conditions used in the stated guarantees. }\label{tab:main}

\scalebox{0.59}{\begin{tabular}{|p{6.2cm}|p{1.8cm}|p{1.6cm}|p{2.0cm}|p{2.0cm}|p{2.2cm}|p{2.29cm}|p{3.9cm}|}
\hline
Algorithm                 & Smoothness Assum.  & Lipschitz-Free & 1st Moment Coef. ($\alpha_t$) $^a$ & 2nd Moment Coef. ($\beta_t$)$^a$ & Stepsize ($\gamma_t$) $^a$ &   Additional Assumptions $^b$    & Convergence Rate  \\
\hline\hline
STORM \cite{cutkosky2019momentum} & Individual & \NO$^c$   & Ada & --- & Ada  & BG, BF & $\widetilde{\mathcal{O}}(1/\sqrt{T} + \sigma^{1/3} T^{-1/3})$ \\
\hline
STORM$^+$ \cite{levy2021storm} & Individual & \YES & Ada& ---& Ada  & BG, BF & ${\mathcal{O}}(1/\sqrt{T} + \sigma^{1/3} T^{-1/3})$ \\
\hline
{\sf SuperADAM-I} \cite{huang2021super} & Individual &  \NO$^c$ &  Const & --- & Open &  BH & $\widetilde{\mathcal{O}}(T^{-1/3})$  \\
\hline
MARS \cite{yuan2025mars}& Individual & \YES  & Open & --- & Open   & BG & $\widetilde{\mathcal{O}}(T^{-1/3})$ \\
\hline
META-STORM \cite{liu2022meta}& Individual & \YES  & Ada & --- & Ada   & BG & $\widetilde{\mathcal{O}}(1/\sqrt{T} + \sigma^{1/3} T^{-1/3})$ \\
\hline
AdaSTORM \cite{JiangAdaStorm}  & Individual & \YES & Const & --- & Const + Ada $^d$ & --- & $\mathcal{O}(T^{-1/3})$ \\
\hline
\hline
{\sf AdaFom} \cite{chen2019convergence} & Average & \NO & Const& Open& Open &  BG & $\widetilde{\mathcal{O}}(T^{-1/4})$  \\
\hline
{\sf Practical Adam} \cite{chen2022towards} & Average & \NO & Open& Const & Open &  BG, $\alpha_t > \beta^2$ & $\widetilde{\mathcal{O}}(T^{-1/4})$  \\
\hline
{\sf SuperADAM-A} \cite{huang2021super} & Average & \NO$^b$ & Const & --- & Open &  BH & $\widetilde{\mathcal{O}}(T^{-1/4})$  \\
\hline
{\sf Vanilla Adam}  \cite{zhang2022adam} & Average & \NO & Open& Const& Open &  BG, $\alpha_t\leq \sqrt{\beta}$ & $\widetilde{\mathcal{O}}(T^{-1/4})$  \\
\hline
{\sf AdaGrad-Norm} \cite{ward2020adagrad} & Average & \YES & --- & --- & Ada $^e$  & BG & $\widetilde{\mathcal{O}}(1/\sqrt{T} + \sigma^{1/2} T^{-1/4})$  \\
\hline
OptEMA-M [ours] & Average & \YES & Ada & Const & Ada  & BG & $\widetilde{\mathcal{O}}(1/\sqrt{T} + \sigma^{1/2} T^{-1/4})$  \\
\hline
OptEMA-V [ours] & Average & \YES & Const & Ada & Ada & BG  & $\widetilde{\mathcal{O}}(1/\sqrt{T} + \sigma^{1/2} T^{-1/4})$   \\
\hline
\end{tabular}}
\vspace{5pt}
\scalebox{0.99}{
{\footnotesize
\renewcommand{\arraystretch}{0.9}
\begin{tabular}{@{}p{\linewidth}@{}}
\textbf{Note a:} Const: fixed constant. Open: open-loop stepsize. Ada: adaptive stepsize using local geometry. \\
\textbf{Note b:}  BF: bounded objective; BG: bounded stochastic gradients a.s.; BH: lower-bounded Hessian eigenvalues. \\
\textbf{Note c:} STORM and SuperADAM are not Lipschitz-free, as they depend on the Lipschitz constant for parameterization. \\
\textbf{Note d:} AdaSTORM uses a hybrid step size (constant + momentum-adapted) and is therefore not fully adaptive. \\
\textbf{Note e:} AdaGrad-Norm uses cumulative scalar normalization rather than EMA-based moment updates. \\
\end{tabular}
}
}

\end{table*}

\section{Preliminaries} \label{sect:prel}
We begin by stating the standard assumptions used throughout the paper.

\begin{assumption}[Lower-Bounded Objective]\label{ass:F}
There exists a constant $f_* > -\infty$ such that $f(\x)\ge f_*$, $\forall\,\x\in\mathbb{R}^n$.

\end{assumption}

\begin{assumption}[Unbiased Stochastic Gradient with Bounded Variance]\label{ass:sigma}
There exists a constant $\sigma \geq 0$ such that, for all $\x\in\mathbb{R}^n$, $\mathbb{E}\!\left[\nabla f(\x;\xi)\right]=\nabla f(\x)$, $\mathbb{E}\!\left[\|\nabla f(\x;\xi)-\nabla f(\x)\|_2^2\right]\le \sigma^2$.
\end{assumption}

\begin{assumption}[Smoothness of the Objective]\label{ass:f}
We assume that $f$ is (average) $L$-smooth, namely, $\|\nabla f(\x)-\nabla f(\y)\| \le L\|\x-\y\|,\,\forall\,\x,\y\in\mathbb{R}^n$.

\end{assumption}

\begin{assumption}[Bounded stochastic gradients] \label{ass:BG}
Let $\g_t:=\nabla f(\x_t;\xi_t)$ and assume
$\EEE[\g_t \mid \FF_{t-1}]=\nabla f(\x_t)$.
There exists a constant $\hat{G}>0$ such that $\|\g_t\|\leq \hG$ almost surely for all $t\geq 1$. Consequently, by Jensen's inequality, $\|\nabla f(\x_t)\| = \|\EEE[\g_t\mid \FF_{t-1}]\| \le \hG$.
\end{assumption}

\begin{remark}
\ding{182} Assumptions \ref{ass:F}--\ref{ass:f} are standard in stochastic optimization; see, e.g., \cite{ghadimi2016mini,cutkosky2019momentum}. Assumption \ref{ass:f} implies the classical descent lemma: $f(\x)\le f(\y)+\la \nabla f(\y),\,\x-\y\ra+\frac{L}{2}\|\x-\y\|_2^2,\, \forall\,\x,\y\in\mathbb{R}^n$, see, for example, Lemma 1.2.3 in \cite{nesterov2003introductory}. \ding{183} Many STORM-type methods adopt the stronger \emph{individual smoothness} condition $\|\nabla f(\x;\xi)-\nabla f(\y;\xi)\| \le L\|\x-\y\|$, $\forall\,\x,\y\in\mathbb{R}^n$, which is strictly stronger than the average smoothness condition in Assumption \ref{ass:f}. \ding{184} In this paper, \emph{Lipschitz-free} means that the algorithm does not use the smoothness constant $L$ for parameter tuning, although the analysis still assumes that $f$ is $L$-smooth. Assumption \ref{ass:BG} is the only boundedness condition retained in the current proof; removing it is an important direction for future work.  

\end{remark}
\paragraph{Notation.}
Unless otherwise specified, vector operations are understood element-wise. For any $\x,\y\in\mathbb{R}^n$, the expressions $\x+\y$, $\x-\y$, $\x\odot\y$, and $\tfrac{\x}{\y}$ denote element-wise addition, subtraction, multiplication, and division, respectively. Additional notation and auxiliary lemmas are collected in Appendix \ref{app:sect:note}.

\section{The OptEMA Algorithms}
\label{sect:proposed}

In this section, we present our unified algorithmic framework, termed {OptEMA} (Exponential Moving Average with zero-noise Optimality), summarized in Algorithm~\ref{alg:main}. The central idea of {OptEMA} is to equip the standard EMA-based update rule with a trajectory-dependent \emph{closed-loop} mechanism driven by the observed optimization path.

To disentangle the roles of the first- and second-moment estimators, we introduce two complementary variants. The first variant, {OptEMA-M}, assigns adaptivity to the first-moment EMA coefficient while keeping the second-moment decay fixed. The second variant, {OptEMA-V}, instead assigns adaptivity to the second-moment EMA coefficient while keeping the first-moment decay fixed. These two variants provide symmetric yet distinct realizations of the same design principle. Here $\alpha_t$ and $\beta_t$ denote the weights assigned to the current stochastic gradient and its squared version, respectively; equivalently, the usual EMA decay factors \cite{KingmaBaCILR15} are $1-\alpha_t$ and $1-\beta_t$.

\begin{figure}[t!]
\refstepcounter{optalg}
\centering
$$
\begin{array}{l}
\hline \mathbf{Algorithm\ \theoptalg}\ \text{OptEMA} \\
\hline
01:\ \textbf{Input: } \text{base learning rate } \theta > 0, \text{initial point } \x_1, \text{parameters } \varepsilon,\alpha_{\text{fix}},\beta_{\text{fix}}\in (0,1], \tau \in [0,1].  \\
02:\ \text{(Default:}\ \tau=1,\,\varepsilon=10^{-5},\,\alpha=0.1,\,\beta=0.001,\ \text{and}\ \theta=1)\\
03:\ \textbf{Initialize: } \x_0=\x_1, \, \m_0 = \mathbf{0}, \, \v_0 = \mathbf{0}, \, \rho_0=\gamma_0=1.\\
04:\ \textbf{for } t = 1\,\textbf{to}\,T\,\textbf{do} \\
05:\ \quad \text{Draw a mini-batch } \xi_t \text{ and compute } \g_t:=\nabla f(\x_t;\xi_t). \\
06:\ \quad \text{Update statistics: } \rho_t = \sqrt{\tfrac{1 + \tfrac{\tau}{t}\sum_{i=1}^t \| \g_i \|^2}{1 + \sum_{i=1}^t \| \g_i \|^2}}. \\
07:\ \quad \text{Set EMA weights: } (\alpha_t, \beta_t) = \begin{cases}
(\rho_t, \beta_{\text{fix}})   & \text{if OptEMA-M} \\
(\alpha_{\text{fix}},\rho_t ) & \text{if OptEMA-V}
\end{cases} \\
08:\ \quad \text{Update EMA: } \m_{t} = (1-\alpha_t) \m_{t-1} + \alpha_t \g_t, \quad \v_{t} = (1-\beta_t) \v_{t-1} + \beta_t \g_t^2. \\
09:\ \quad \text{Compute update stepsize: } \gamma_t = \begin{cases}
\min\big(\alpha_t, \sqrt{\alpha_t} \big(1 + \sum_{j=1}^t \|\m_j\|^2 \big)^{-1/2}  \big)  & \text{if OptEMA-M} \\
\big(1 + \sum_{j=1}^t \|\m_j\|^2 \big)^{-1/2} & \text{if OptEMA-V}
\end{cases} \\
10:\ \quad \text{Update parameters: } \x_{t+1} = \x_t - \theta \gamma_t \cdot \frac{\m_{t}}{\varepsilon+ \sqrt{\v_{t}}}. \\
11:\ \textbf{end for} \\
\hline
\end{array}
$$
\label{alg:main}
\end{figure}

\subsection{A Unified Closed-Loop Framework}

The goal of {OptEMA} is to construct a closed-loop adaptive stepsize mechanism while preserving the canonical exponential moving average structure used in Adam-type methods.

At iteration $t$, the algorithm computes a stochastic gradient $\g_t := \nabla f(\x_t;\xi_t)$ from one sample or mini-batch and forms the accumulated gradient energy and the feedback coefficient:
\beq
\ts \rho_t = \sqrt{\frac{1 + \frac{\tau}{t}\sum_{i=1}^t \| \g_i \|^2}{1 + \sum_{i=1}^t \| \g_i \|^2}},
\eeq
\noi where \(\tau \in [0,1]\). We refer to \(\rho_t\), with the default choice \(\tau=1\), as the \emph{Corrected AdaGrad-Norm coefficient schedule}. To see its relation to standard AdaGrad-Norm \cite{ward2020adagrad}, define $a_t := (1+G_t)^{-1/2}$, where $G_t := \sum_{i=1}^t \|g_i\|^2$. Then $\rho_t = \left(1+\frac{\tau G_t}{t}\right)^{1/2} a_t$. Thus, \(\rho_t\) can be interpreted as the standard AdaGrad-Norm coefficient multiplied by a historical-gradient averaging correction. The denominator \(1+G_t\) keeps the usual self-normalized suppression effect, while the numerator $1+\tau G_t/t$ offsets excessive decay when the accumulated gradient energy is dominated by a few large past gradients.

This dynamic schedule has three advantages. \ding{182} \textbf{Noise-adaptive suppression:} It guarantees noise-adaptive suppression in high-noise settings (where $\rho_t = \mathcal{O}(1/\sqrt{t})$), yet stabilizes to a constant near stationary points when gradients vanish, ensuring zero-noise optimality. \ding{183} \textbf{Robustness to Gradient Spikes:} In the presence of occasional gradient spikes, the numerator partially offsets the impact and reduces the immediate suppression of subsequent stepsizes, while the cumulative denominator still records historical gradients. \ding{184} \textbf{Dimensional Consistency:} The construction maintains matching dimensions between the numerator and the denominator, providing a unified foundation for subsequent noise-adaptive analyses.

These desirable properties are rigorously formalized in the following lemma, which establishes useful bounds and the monotonic non-increase of $\rho_t$.

\begin{lemma} \label{lemma:corrected} (Proof in Appendix \ref{app:lemma:corrected}) Let $\rho_t = \sqrt{ (1 + \tfrac{\tau}{t}\GGG_t) / (1 + \GGG_t)}$, where $\GGG_t:=\sum_{i=1}^t \|\g_i\|^2$ and $\tau\in[0,1]$. Define $\rho_0=1$ and $\hat{g}_t = \max_{1\le i\le t} \|\g_i\|$. For all $t\geq 1$, we have: \textbf{(a)} $\sqrt{\frac{\tau}{t}}\le \rho_{t}\le 1$. \textbf{(b)} $1\le \tfrac{\rho_{t-1}}{\rho_{t}} \leq \sqrt{ 1 + {\hat g_{t}^2}/(1 + \GGG_{t-1})}$. \textbf{(c)} $\rho_{t}^2 \le \tfrac{1 + \hat{g}_{t}^2}{1 + \GGG_t}$. \textbf{(d)} $(\rho_{t-1}-\rho_t)^2/\rho_t  \le \tfrac{\hat{g}_{t}^3}{1 + \GGG_t}$.

\end{lemma}

In summary, $\rho_t$ captures the cumulative gradient magnitude along the optimization path and directly drives the adaptive EMA coefficient in the two OptEMA variants. The quantity $\hat g_t$ records the largest gradient norm observed so far and is used mainly as an analytical device for bounding the variation of $\rho_t$.

The algorithm then maintains the standard first- and second-moment EMA estimates
\[
\ts \m_t = (1-\alpha_t)\m_{t-1} + \alpha_t \g_t,
\qquad
\v_t = (1-\beta_t)\v_{t-1} + \beta_t \g_t^2,
\]
which retain exactly the same structural form as in Adam-style optimizers.
The essential departure from standard Adam lies in the fact that the coefficients $(\alpha_t,\beta_t)$ and the effective stepsize $\gamma_t$ are no longer preset in an open-loop manner; instead, they are determined adaptively from the observed trajectory.

Within this framework, we study two complementary designs. One places adaptivity on the first-moment coefficient $\alpha_t$, and the other places adaptivity on the second-moment coefficient $\beta_t$.
Both variants share the same EMA backbone and the same closed-loop philosophy, but they emphasize different aspects of moment estimation and stability control.

\subsection{OptEMA-M: Adaptive First-Moment Averaging}

The first variant, {OptEMA-M}, introduces adaptivity through the first-moment coefficient $\alpha_t$. Specifically, we set
\beq
\ts \alpha_t = \rho_t, \qquad \beta_t = \beta_{\text{fix}} \in (0,1], \qquad
\gamma_t = \min \left(\alpha_t, \sqrt{\alpha_t} \left( 1 + \sum_{j=1}^t \|\m_j\|^2 \right)^{-1/2}\right).
\eeq
Under this construction, the first-moment EMA coefficient decreases automatically with the accumulated gradient magnitude, whereas the second-moment estimator remains a standard EMA with fixed decay.
As a result, the influence of newly observed gradients is gradually attenuated as the trajectory evolves, which leads to a progressively more stable momentum estimate.

From a theoretical viewpoint, adapting $\alpha_t$ creates a stronger coupling between the momentum tracking error and the stepsize dynamics. For this reason, the closed-loop stepsize $\gamma_t$ contains two complementary components: a coefficient cap $\alpha_t$ and an energy-control term $\sqrt{\alpha_t}(1+\sum_{j=1}^t\|\m_j\|^2)^{-1/2}$. This design is crucial for closing the Lyapunov-based analysis.

\subsection{OptEMA-V: Adaptive Second-Moment Averaging}

The second variant, {OptEMA-V}, instead introduces adaptivity through the second-moment coefficient. More precisely, we choose
\beq
\ts \alpha_t = \alpha_{\text{fix}} \in (0,1],\qquad \beta_t = \rho_t,\qquad\gamma_t = \left(1+\sum_{j=1}^t \|\m_j\|^2\right)^{-1/2}.
\eeq

Under this scheme, the first-moment coefficient remains fixed, while the second-moment coefficient is adapted through $\beta_t=\rho_t$, where $\rho_t$ is computed from the accumulated squared gradient norms. Consequently, {OptEMA-V} places greater emphasis on trajectory-adaptive second-moment averaging within the EMA framework.

The corresponding closed-loop stepsize $\gamma_t$ is again fully trajectory-dependent.
The factor $\big(1+\sum_{j=1}^t \|\m_j\|^2\big)^{-1/2}$ modulates the effective stepsize according to the cumulative momentum energy. Together, these two components yield a fully adaptive update rule that does not require prior knowledge of any Lipschitz constant.

\section{Convergence Guarantees}
\label{sect:rate}

In this section, we establish the convergence guarantees of the proposed {OptEMA} algorithms.

\paragraph{Notation.} Throughout this section, we define $\s_t \triangleq \m_t-\nabla f(\x_t)$, and $\hat{g}_t = \max_{i=1}^t \| \g_i \|$. We let $T\geq 1$ denote the total number of iterations. Denote $\GGG_t := \sum_{i=1}^t \|\g_i\|^2$ and $\MMM_t := \sum_{i=1}^t \|\m_i\|^2$. For brevity, we drop the subscript when $t=T$, denoting $\GGG := \GGG_T$ and $\MMM := \MMM_T$. We use $\mathbb{E}[\cdot]$ for the total expectation and
$\mathbb{E}_t[\cdot]:=\mathbb{E}[\cdot\mid \mathcal{F}_{t-1}]$
for the conditional expectation. We denote $\FF_{t-1}:=\sigma(\xi_1,\ldots,\xi_{t-1})$ as the natural filtration, and $t$ can be inferred from the context.



\subsection{Analysis for OptEMA-M}
\label{sect:it:EMA:M}

In this section, we establish the theoretical convergence rate of the OptEMA-M algorithm.

The proof proceeds in three steps. We begin by establishing basic bounds on the internal EMA sequences, then control the gradient tracking error, and finally combine these ingredients with a one-step descent inequality to derive the global convergence rate.

We start with two elementary estimates that relate the momentum sequence $\m_t$ to the stochastic gradients $\g_t$ and control the second-moment estimator together with the adaptive EMA coefficient.

\begin{lemma} \label{lemma:bound:m:using:g:1}
(Proof in Appendix \ref{app:lemma:bound:m:using:g:1}) For all $T\geq 1$, we have: $\sum_{t=1}^T \|\m_t\|^2 \le \kappa \sum_{t=1}^T \|\g_t\|^2 $, where $\kappa:=2(1 + \sqrt{\tau} \hG) $.

\end{lemma}

\begin{lemma} \label{lemma:bound:EMA:M:v}
(Proof in Appendix \ref{app:lemma:bound:EMA:M:v}) For all $t\geq 1$, we have: $\|\v_t\| \le {\hG}^2$, and $\|\s_t\| \le 2 {\hG}$.
\end{lemma}

A central quantity in the analysis is the gradient tracking error $\s_t := \m_t - \nabla f(\x_t)$, which measures the deviation of the first-moment estimator from the true gradient. 

For OptEMA-M, the coefficient $\alpha_t=\rho_t$ depends on the current stochastic gradient and is therefore generally not $\mathcal F_{t-1}$-measurable. The next lemma gives a recursive bound on its expected squared norm.

\begin{lemma}\label{lemma:bound:eee:1}
(Proof in Appendix \ref{app:lemma:bound:eee:1})
Let $\lambda_1:=\frac{3\theta^2 L^2}{\varepsilon^2}(1+{\hG})$, $\lambda_2:=(1+\hG^2+10 \hG^3)(4+2 \hG^2)$. Use the conventions $\x_0=\x_1$, $\m_0=\v_0=\zero$, $\alpha_0=\gamma_0=1$, and let $\FF_0:=\sigma(\x_1)$. For any $t\ge1$ and any nonnegative integrable $\FF_{t-1}$-measurable random variable $H_{t-1}$, we have
\beq
\EEE\!\left[H_{t-1}\|\s_t\|^2\right] \le \EEE\!\left[ H_{t-1}\left(1-\tfrac{\alpha_t}{4}\right)\|\s_{t-1}\|^2 + H_{t-1}B_t \right],
\eeq
\noi where $B_t:=\lambda_1\frac{\gamma_{t-1}^2}{\alpha_{t-1}}\|\m_{t-1}\|^2+\lambda_2\frac{\|\g_t\|^2}{1+\sum_{i=1}^t\|\g_i\|^2}$.


\end{lemma}

We next derive the basic descent relation for a single iteration. This inequality isolates the contribution of the true gradient from the errors induced by momentum tracking and adaptive scaling.

\begin{lemma}[One-Step Descent Inequality]\label{lemma:OptEMA:M:suff:dec}
(Proof in Appendix \ref{app:lemma:OptEMA:M:suff:dec}) Define $c_0:=\tfrac{\theta}{ 2(1+{\hG})}$, $c_1:=\tfrac{\theta}{ 2\varepsilon^2} (1+{\hG})$, and $c_2:=\tfrac{L\theta^2}{2\varepsilon^2}$. For all $t\geq 1$, the following inequalities hold almost surely:
\beq
c_0 \gamma_t \|\nabla f(\x_t)\|^2  +  f(\x_{t+1}) - f(\x_{t})  &\leq&  c_1 \gamma_t \| \s_t \|^2    + c_2 \gamma_t^2 \|\m_{t}\|^2,  \label{eq:OptEMA:M:suff:dec} \\
c_0\gamma_t\|\m_t\|^2 +  f(\x_{t+1}) - f(\x_{t})  &\leq& c_1\gamma_t\|\s_t\|^2 + c_2\gamma_t^2\|\m_t\|^2.  \label{eq:OptEMA:M:suff:dec:2}
\eeq

\end{lemma}

To telescope Inequality (\ref{eq:OptEMA:M:suff:dec}) or (\ref{eq:OptEMA:M:suff:dec:2}) over $t=1,\dots,T$, it remains to control the cumulative error terms on the right-hand side. We first bound the weighted tracking error, and then establish logarithmic growth bounds for the objective values and stochastic gradients.

\begin{lemma} \label{lemma:sum:tail} (Proof in Appendix \ref{app:lemma:sum:tail}) Define $C_b:=(\lambda_1+\lambda_2) \cdot (1+\ln \kappa  ) \cdot \ln(e+ \hG^2)$ and $C_c:=4 (1+\hG)\cdot \hG^2 + 4 (1+\hG) C_b$. Then, for all $T\geq1$, $\sum_{t=1}^T B_t\leq C_b \ln (e+T)$ a.s., and $\EEE\left[\sum_{t=1}^T \gamma_t \|\s_t\|^2\right] \le \EEE\left[\sum_{t=1}^T \alpha_t \|\s_t\|^2\right] \leq C_c \ln(e+T)$.

\end{lemma}

\begin{lemma}[Expected Logarithmic Objective Growth]\label{lemma:OptEMAM:G:bound}(Proof in Appendix \ref{app:lemma:OptEMAM:G:bound}) Define $D_T:=\sum_{t=1}^T \gamma_t^2\|\m_t\|^2$. Define $C_d:=\ln(e+\kappa \hG^2)$, and $C_f:= f(\x_1)-f_*+ c_1 C_c + c_2 C_d $. Then, for all $T\geq 1$, $D_T \leq C_d \ln(e+T)$ a.s., and $\EEE[f(\x_T)-f_*]\leq C_f\ln(e+T)$.

\end{lemma}

Instead of dividing the descent inequality by the random stepsize and telescoping random weighted objective values, we first prove that the weighted tracking error has a logarithmic second-moment bound. The proof uses the predictable-multiplier tracking recursion together with a deterministic look-ahead telescoping lemma; the almost-sure descent inequality is then used to convert this tracking control into a momentum-energy bridge.

\begin{lemma}[Squared Weighted Tracking Error]\label{lemma:M:U2} (Proof in Appendix \ref{app:lemma:M:U2}) Define $U_T:=\sum_{t=1}^T\gamma_t\|\s_t\|^2$. Then, for all $T\ge1$, $\EEE[U_T^2]\le C_{u} \ln^2(e+T)$, where $C_u := C_c (4 \hG^2 + 8 C_b + 32  \hG^2)$.
\end{lemma}

\begin{lemma}[Weighted momentum-energy bridge]\label{lemma:M:weighted:m}
(Proof in Appendix \ref{app:lemma:M:weighted:m}) Let $C_v:= 2 C_w + \tfrac{6}{c_0^2} ( (f(\x_1) - f_*)^2+ c_1^2 C_u + c_2^2 C_d^2 )$, where $C_w:=\tfrac{1}{c_0}\big( \ts f(\x_1) - f_* + c_1 C_c + c_2 C_d \big)$. Then, for all $T\geq1$, $\EEE[\alpha_T\MMM_T]\le C_{v} \ln^2(e+T)$.
\end{lemma}

Using the above expected logarithmic bounds, we can further control the cumulative tracking error and the cumulative true gradient norm in terms of the path-dependent quantities $\GGG:=\sum_{t=1}^T \|\g_t\|^2$ and $\MMM:=\sum_{t=1}^T \|\m_t\|^2$.

\begin{lemma} \label{lemma:M:sum:m}
(Proof in Appendix \ref{app:lemma:M:sum:m}) Let $\ZZ_T:=\EEE\left[\sqrt{1+\GGG_T}\right]$, and $C_z:=\sqrt{2} (\sqrt{C_v} + \sqrt{C_c })$. Then, for all $T\geq1$, $\EEE[\sqrt{\sum_{t=1}^T\|\nabla f(\x_t)\|^2}] \le C_{z} \ZZ_T^{1/2} \cdot \ln(e+T)$.
\end{lemma}

We are now ready to state the main convergence guarantee for {OptEMA-M}, obtained by combining the preceding bounds and taking expectation over the full trajectory.

\begin{theorem} \label{the:it:EMA:M}
(Proof in Appendix \ref{app:the:it:EMA:M})
Under Assumptions \ref{ass:F}--\ref{ass:BG}, we have
$$\ts \EEE\left[\frac{1}{T}\sum_{t=1}^T \| \nabla f(\x_{t})\|\right] \leq \OO(T^{-1/2} \ln^2(e+T) + \sqrt{\sigma} T^{-1/4} \ln(e+T)).$$
\noi Consequently, if $R \sim \mathrm{Unif}\{1,2,\dots,T\}$ is independent of the algorithmic randomness and $\x_{\rm out}:=\x_R$, then the same upper bound holds for
$\EEE[\|\nabla f(\x_{\rm out})\|]$.
\end{theorem}

\begin{remark}
\ding{182} Theorem \ref{the:it:EMA:M} shows that {OptEMA-M} attains the unified noise-adaptive rate $\widetilde{\mathcal{O}}(T^{-1/2}+\sigma^{1/2}T^{-1/4})$ for the averaged gradient norm. \ding{183} Under Assumptions \ref{ass:F}--\ref{ass:BG}, when $\sigma=0$, the variance-dependent term vanishes and the bound reduces to the nearly deterministic rate $\widetilde{\mathcal{O}}(T^{-1/2})$ for the averaged gradient norm. The analysis assumes bounded stochastic gradients in order to control the adaptive normalizers and pathwise tracking terms.

\end{remark}

\subsection{Analysis for OptEMA-V}
\label{sect:it:EMA:V}

In this subsection, we establish the theoretical convergence guarantees for the OptEMA-V algorithm.

As in the analysis of {OptEMA-M}, we first control the internal algorithmic states, then decouple the historical terms induced by the EMA recursion, and finally derive the global convergence rate. We begin with two basic bounds on the momentum and second-moment sequences.

\begin{lemma} \label{lemma:bound:m:using:g:2}
(Proof in Appendix \ref{app:lemma:bound:m:using:g:2}) Assume that $\{\alpha_t\}_{t\geq 1}\subset(0,1]$ is a deterministic and non-increasing sequence of real numbers. For any integer $T\ge 1$, it holds that $\sum_{t=1}^T \|\m_t\|_2^2 \le \tfrac{\alpha_1}{\alpha_T}\sum_{t=1}^T \|\g_t\|_2^2$.
\end{lemma}

\begin{lemma} \label{lemma:bound:v}
(Proof in Appendix \ref{app:lemma:bound:v}) For all $t\geq 1$, we have: $\|\v_t\| \le {\hG}^2$, $\|\m_t\| \le {\hG}$, and $\gamma_t\leq \gamma_{t-1}\leq (1+{\hG}) \gamma_t$.
\end{lemma}

We next establish a one-step descent inequality. Unlike {OptEMA-M}, the analysis here involves a nonlocal historical term generated by the EMA structure, so the descent at iteration $t$ depends on the entire past trajectory.

\begin{lemma}[One-Step Descent Inequality]\label{lemma:OptEMA:V:suff:dec}
(Proof in Appendix \ref{app:lemma:OptEMA:V:suff:dec})
Assume that $\{\alpha_t\}_{t\geq 1}\subset(0,1]$ is a deterministic and non-increasing sequence of real numbers. Define $c_0:=\frac{\theta {\hG}^2}{\varepsilon}$, $c_1:=\frac{\theta^2 L}{\varepsilon^2}$, $c_2:=\frac{2\theta(\varepsilon+{\hG}){\hG}^2}{\varepsilon^4}$, $c_3:=\frac{L\theta^2}{2\varepsilon^2}$, $c_4:=\frac{\theta}{2(\varepsilon+{\hG})}$. For $1\leq j\leq t$, define $\Pi_{j,t}:=\prod_{k=j+1}^{t}(1-\alpha_k)$, with the convention that $\Pi_{t,t}=1$. Then, for all $t\geq 1$, we have
\beq \label{eq:OptEMA:V:suff:dec}
\EEE[ f(\x_{t+1}) - f(\x_t)]
&\leq&
\ts
\EEE\Big[
c_3\gamma_t^2\|\m_t\|^2
+
\sum_{j=1}^{t}
\Pi_{j,t}(E_j-D_j)
\nn\\
&&\ts
+
c_0(\gamma_{t-1}-\gamma_t)
+
c_0
\sum_{j=2}^{t}
\Pi_{j,t}(1-\alpha_j)(\gamma_{j-2}-\gamma_{j-1})
\Big].
\eeq
Here, $E_t := c_1\gamma_{t-1}^2 \|\m_{t-1}\|^2 + c_2 \tfrac{\rho_t}{\alpha_t}
\gamma_{t-1}\|\m_t\|^2$, and $D_t := c_4 \alpha_t \gamma_{t-1} \|\nabla f(\x_t)\|^2$.

\end{lemma}

The historical summation in \eqref{eq:OptEMA:V:suff:dec} couples the error term $E_j$ and the descent term $D_j$ across time. To telescope this inequality over all iterations, we need to decouple such temporal dependence. The next lemma provides precisely this mechanism by controlling accumulated EMA-weighted sums.

\begin{lemma}[EMA Accumulation Bound]\label{lemma:alpha:b:gamma}
\textit{(Proof in Appendix \ref{app:lemma:alpha:b:gamma})} Assume $A_j\geq 0$ for all $j \geq 1$. If $\alpha_t\in(0,1]$ is a deterministic and non-increasing sequence, we have:
\beq
\ts \sum_{t=1}^T A_t \leq \sum_{t=1}^T \sum_{j=1}^{t}  \left( \prod_{k=j+1}^t (1-\alpha_k) \right) A_j \leq \frac{1}{\alpha_T} \sum_{t=1}^T A_t. \nn
\eeq
\end{lemma}

Before concluding the proof, we also need expected logarithmic bounds on the objective values and stochastic gradients.

\begin{lemma}[Expected Logarithmic Objective Growth]\label{lemma:OptEMAV:G:bound}
(Proof in Appendix \ref{app:lemma:OptEMAV:G:bound}) Assume that $\alpha_t=\alpha\in(0,1]$. Define $C_e:=\frac{c_1}{\alpha} + (1+\hG)^2 \cdot \frac{c_2}{\alpha^2}$, and $C_f:= f(\x_1) - f_* + \frac{c_0}{\alpha} + \big(c_3 + C_e \big) \cdot \ln(e+\hG^2)$. Then, for all $t\geq1$, $\sum_{t=1}^{T}\sum_{j=1}^{t}(1-\alpha)^{t-j}E_j\leq C_e \ln(e+\GGG_T)$ almost surely, $\EEE[f(\x_t)-f_*]\leq C_{f}\ln(e+t)$.
\end{lemma}

The next lemma controls a weighted true-gradient energy and then converts it to the unweighted quantity by the adaptive momentum norm.

\begin{lemma}[Weighted true-gradient energy bridge]\label{lemma:bound:sum:nabla}
(Proof in Appendix \ref{app:lemma:bound:sum:nabla}) Define $C_h:= \frac{1}{c_4\alpha}\cdot (C_e + f(\x_1) - f_* + c_3 + \tfrac{c_0}{\alpha}) \ln (e+\hG^2)$, and $C_z:=\sqrt{C_h}$. Then, for all $T\geq1$, $\EEE[\sqrt{\sum_{t=1}^T\|\nabla f(\x_t)\|^2}]
\leq C_z\ZZ_T^{1/2}\cdot\sqrt{\ln(e+T)}$.
\end{lemma}

Closing this bridge with the noise decomposition gives the following explicit convergence bound for {OptEMA-V}.

\begin{theorem} \label{the:it:EMA:V}
(Proof in Appendix \ref{app:the:it:EMA:V}) Under Assumptions \ref{ass:F}--\ref{ass:BG}, assume that $\alpha_t = \alpha \in (0,1]$.  We have
\beq
\ts \EEE\left[\frac{1}{T}\sum_{t=1}^T \| \nabla f(\x_{t})\|\right]
\leq\ts \OO(T^{-1/2} \ln(e+T) + \sqrt{\sigma} T^{-1/4} \ln^{1/2}(e+T)).
\eeq
Consequently, if $R \sim \mathrm{Unif}\{1,2,\dots,T\}$ is independent of the algorithmic randomness and
$\x_{\rm out}:=\x_R$, then the same upper bound holds for
$\EEE[\|\nabla f(\x_{\rm out})\|]$.

\end{theorem}

\begin{remark} 
\ding{182} Theorem \ref{the:it:EMA:V} shows that {OptEMA-V} also attains the noise-adaptive rate 
$\widetilde{\mathcal{O}}(T^{-1/2}+\sigma^{1/2}T^{-1/4})$ under the stated average-smoothness and bounded stochastic-gradient assumptions. 
\ding{183} Compared with {OptEMA-M}, the proof of {OptEMA-V} involves a more delicate treatment of the second-moment recursion, which is reflected in a different polylogarithmic dependence. Nevertheless, the resulting guarantee is trajectory-dependent and horizon-free.
\end{remark}

\section{Conclusion} \label{sect:conc}

We proposed OptEMA, a closed-loop family of adaptive exponential moving average methods for stochastic nonconvex optimization. The proposed framework preserves the standard Adam-style EMA update structure while replacing open-loop coefficient schedules with trajectory-dependent feedback. The key mechanism is a Corrected AdaGrad-Norm coefficient schedule, which combines self-normalized suppression with a historical-gradient averaging correction. Under lower-boundedness, unbiasedness, bounded variance, average smoothness, and bounded stochastic gradients, we proved that both OptEMA-M and OptEMA-V achieve the unified noise-adaptive rate $\widetilde{\mathcal O}\bigl(T^{-1/2}+\sigma^{1/2}T^{-1/4}\bigr)$ for the averaged gradient norm. When \(\sigma=0\), this guarantee automatically reduces to the nearly deterministic rate \(\widetilde{\mathcal O}(T^{-1/2})\), showing that EMA-based adaptive methods can recover zero-noise optimality without modifying the basic EMA update structure. The current proof still relies on an almost-sure bounded stochastic-gradient condition to control the adaptive normalizers and pathwise tracking quantities. Removing this assumption while preserving the closed-loop EMA structure is the main theoretical challenge left open by this work. Another promising direction is to extend the corrected EMA schedule to coordinate-wise, layer-wise, or matrix-valued adaptive normalizers.


\clearpage
\normalem
\bibliographystyle{plain}
\bibliography{mybib}

\clearpage
\appendix

{\huge Appendix}

\noi The appendix is structured as follows:

\noi {Appendix \ref{app:sect:note}} introduces the notation and collects several supporting lemmas.

\noi {Appendix \ref{app:sect:proposed}} presents the proofs of the results in Section \ref{sect:proposed}.

\noi {Appendix \ref{app:sect:it:EMA:M}} presents the proofs of the results in Section \ref{sect:it:EMA:M}.

\noi {Appendix \ref{app:sect:it:EMA:V}} presents the proofs of the results in Section \ref{sect:it:EMA:V}.


\section{Notations and Supporting Lemmas}\label{app:sect:note}

\subsection{Notations}

Throughout the paper, we use the following notation. Bold lowercase letters denote vectors, while uppercase letters denote real-valued matrices.

\begin{itemize}[leftmargin=12pt,itemsep=0.2ex]

\item $[n]$: The set $\{1,2,...,n\}$.

\item $\|\mathbf{x}\|$: Euclidean norm, defined as $\|\mathbf{x}\|=\|\mathbf{x}\|_2 = \sqrt{\la \mathbf{x},\mathbf{x}\ra}$.

\item $\la \mathbf{a},\mathbf{b}\ra$ : Standard inner product, given by $\la \mathbf{a},\mathbf{b}\ra =\sum_{i}{\mathbf{a}_{i}\mathbf{b}_{i}}$.



\item $\a\leq \alpha$: For $\a\in\Rn^n$ and scalar $\alpha\in\Rn$, this means $\a_i\leq \alpha$ for all $i\in[n]$.

\item $\EEE[v]$: Expectation of the random variable $v$.

\item $\{A_i\}_{i=1}^{\infty}$, $\{B_i\}_{i=1}^{\infty}$: sequences indexed by non-negative integers.

\item $\text{Diag}(\x)$: the $n \times n$ diagonal matrix with $\x \in \mathbb{R}^n$ on its main diagonal.






\end{itemize}


\subsection{Supporting Lemmas}

We collect several useful lemmas that will be used in the subsequent analysis. These results are independent of the specific context of the paper.



\begin{lemma}
\label{lemma:adaptive:important}
Let $b_1,b_2,\ldots,b_T \ge 0$, and let $p\in(0,1)$. Then the following inequalities hold:

\begin{enumerate}[label=\textbf{(\alph*)}, leftmargin=18pt, itemsep=2pt, topsep=1pt, parsep=0pt, partopsep=0pt]

\item $\sum_{t=1}^T \frac{b_t}{ 1 + \sum_{i=1}^t b_i }  \leq \ln\left(1+ \sum_{t=1}^T b_t \right)$.

\item $\sum_{t=1}^T \frac{b_t}{\left( 1 + \sum_{i=1}^{t} b_i  \right)^{p}} \leq \frac{1}{1-p} \cdot \left(1 + \sum_{t=1}^T b_t \right)^{1-p}$.

\end{enumerate}

\begin{proof}
The inequality in part (a) generalizes Lemma 3.2 of \cite{ward2020adagrad}, and the inequality in part (b) generalizes Lemma 3 of \cite{levy2021storm}. Both statements follow from a routine induction argument, and we therefore omit the proof.
\end{proof}

\end{lemma}


\begin{lemma}\label{lemma:sum_by_parts}
Let $\{A_t\}_{t=1}^{T}$ and $\{B_t\}_{t=1}^{T+1}$ be two nonnegative sequences, with $\{A_t\}_{t=1}^T$ being non-decreasing. Then we have: $\sum_{t=1}^T A_t (B_t - B_{t+1}) \leq A_T \Big(\max_{1 \leq t \leq T} B_t\Big)$.

\begin{proof}
Using summation by parts and the nonnegativity of $A_T$ and $B_{T+1}$, we have:
\beq
\ts \sum_{t=1}^T A_t (B_t - B_{t+1}) &= & \ts A_1 B_1 - A_T B_{T+1} + \sum_{t=2}^T (A_t - A_{t-1}) B_t \nn \\
&\leq & \ts A_1 B_1 + \sum_{t=2}^T (A_t - A_{t-1}) B_t \nn\\
&\leq & \ts A_1 B_1 + \big( \max_{t=2}^T B_t \big) \big( \sum_{t=2}^T A_t - A_{t-1}\big) \nn\\
& = & \ts A_1 B_1 + \big( \max_{t=2}^T B_t \big) \big( A_T - A_{1} \big) \nn\\
& \leq & \ts \big( \max_{t=1}^T B_t \big) A_T.  \nn
\eeq

\end{proof}
\end{lemma}

\begin{lemma}
 \label{lemma:self:normalized:noise}
[Self-normalized Noise Domination] Let $\g_t$ be a stochastic gradient satisfying $\EEE[\g_t\mid \FF_{t-1}]=\nabla f(\x_t)$ and $\|\g_t\|\le {\hG}$ almost surely. Define $\GGG_t:=\sum_{i=1}^t\|\g_i\|^2$. Then the following conditional inequality holds:
\[
\ts \EEE \left[
\frac{\|\g_t-\nabla f(\x_t)\|^2}{1+\GGG_t}
\mathrel{\Big|}\FF_{t-1}
\right]
\le
(4+2 \hG^2)
\EEE \left[
\frac{\|\g_t\|^2}{1+\GGG_t}
\mathrel{\Big|}\FF_{t-1}
\right].
\]
Consequently, for any nonnegative integrable $\FF_{t-1}$-measurable random variable $H_{t-1}$,
\[
\ts \EEE \left[H_{t-1}
\frac{\|\g_t-\nabla f(\x_t)\|^2}{1+\GGG_t}
\right]
\le
(4+2 \hG^2)
\EEE \left[H_{t-1}
\frac{\|\g_t\|^2}{1+\GGG_t}
\right].
\]

\begin{proof}
Define $A_t:=1+\sum_{i=1}^{t-1}\|\g_i\|^2$. Conditioning on $\FF_{t-1}$, both $A_t$ and $\nabla f(\x_t)$ are fixed. For brevity, denote $\EEE_t[\cdot]:=\EEE[\cdot\mid \FF_{t-1}]$.

Therefore, we derive the following inequalities:
\beq
\ts \EEE_t \left[
\frac{\|\g_t-\nabla f(\x_t)\|^2}{A_t+\|\g_t\|^2}
\right] &\overset{\step{1}}{\leq} & \ts 2\EEE_t \left[
\frac{\|\g_t\|^2}{A_t+\|\g_t\|^2}
\right] + 2\|\nabla f(\x_t)\|^2
\EEE_t \left[
\frac{1}{A_t+\|\g_t\|^2}
\right] \nn\\
&\overset{\step{2}}{\leq} & \ts 2\EEE_t \left[
\frac{\|\g_t\|^2}{A_t+\|\g_t\|^2}
\right] + 2 \EEE_t[\|\g_t\|^2] \cdot
\frac{1}{A_t}  \nn\\
&\overset{\step{3}}{\leq} & \ts 2\EEE_t \left[
\frac{\|\g_t\|^2}{A_t+\|\g_t\|^2}
\right] + 2 \EEE_t \left[ (1+\tfrac{{\hG}^2}{A_t})\cdot
\tfrac{\|\g_t\|^2}{A_t+\|\g_t\|^2} \right]   \nn\\
&\overset{\step{4}}{\leq} & \ts 2\EEE_t \left[
\frac{\|\g_t\|^2}{A_t+\|\g_t\|^2}
\right] + 2 \EEE_t \left[ (1+ \hG^2 )\cdot
\tfrac{\|\g_t\|^2}{A_t+\|\g_t\|^2} \right]   \nn\\
&\overset{}{=} &  \ts (4 + 2 \hG^2) \EEE_t \left[
\frac{\|\g_t\|^2}{A_t+\|\g_t\|^2}
\right],  \nn
\eeq
\noi where step \step{1} uses the Young's inequality that $\|\g_t-\nabla f(\x_t)\|^2
\le 2\|\g_t\|^2+2\|\nabla f(\x_t)\|^2$; step \step{2} uses $\|\nabla f(\x_t)\|^2 =\|\EEE_t[\g_t]\|^2 = \|\EEE_t[\g_t]\|^2
\le
\EEE_t[\|\g_t\|^2]$, and $\EEE_t \left[
\frac{1}{A_t+\|\g_t\|^2}
\right]
\le
\frac{1}{A_t}$; step \step{3} uses $\|\g_t\|\le {\hG}$, leading to the pointwise bound $\frac{\|\g_t\|^2}{A_t}
\le
\left(1+\frac{{\hG}^2}{A_t}\right)
\frac{\|\g_t\|^2}{A_t+\|\g_t\|^2}$; step \step{4} uses $A_t\geq 1$.

\end{proof}

\end{lemma}

\section{Proofs for Section \ref{sect:proposed}}
\label{app:sect:proposed}

\subsection{Proof of Lemma \ref{lemma:corrected}}
\label{app:lemma:corrected}

\begin{proof} Define $\GGG_t=\sum_{i=1}^t \|\g_i\|^2$ and $\hat{g}_t = \max_{1\le i\le t} \|\g_i\|$.

\noi Define $\rho_t := \sqrt{\frac{1 + \frac{\tau}{t}\sum_{i=1}^t \| \g_i \|^2}{1 + \sum_{i=1}^t \| \g_i \|^2}}$, where $\tau \in [0,1]$.

\noi For any $x\geq 0$, we let $\phi_t(x):=\frac{1+\frac{\tau}{t}x}{1+x}$.

\paragraph{Part (a).} For all $t\geq 1$, we derive:
\beq
\ts \rho_t^2 - \frac{\tau}{t} \,= \,\frac{1 + \frac{\tau}{t} \GGG_t }{1 + \GGG_t } - \frac{\tau}{t} \,=\, \frac{1-\tau/t}{1+\GGG_t} \,\ge\, 0.\nn
\eeq
\noi where the inequality uses $t \ge 1$ and $\tau \in [0,1]$. This leads to $\rho_t \ge \sqrt{\frac{\tau}{t}}$. Additionally, we trivially have $1+\frac{\tau}{t}\GGG_t \le 1+\GGG_t$, which yields $\rho_t \le 1$. 

\paragraph{Part (b-i).} We now show $\rho_t \le \rho_{t-1}$ for $t \ge 2$. We obtain
\beq
\rho_t^2 ~\overset{\step{1}}{=} \, \phi_t(\GGG_t) \,\overset{\step{2}}{\le} \,\phi_t(\GGG_{t-1}) \,\overset{\step{3}}{\le}\, \phi_{t-1}(\GGG_{t-1}) \,\overset{\step{4}}{=}\, \rho_{t-1}^2,\nn
\eeq
\noi where step \step{1} uses the definition of $\phi_t(x)$; step \step{2} uses $\GGG_t \ge \GGG_{t-1}$, and the fact that $\phi_t(\cdot)$ is a non-increasing function since $\phi_t'(x) = \frac{\tau/t - 1}{(1+x)^2} \le 0$; step \step{3} uses $\frac{\tau}{t} \le \frac{\tau}{t-1}$. Additionally, for the case where $t=1$, we have $\rho_1-\rho_0=\sqrt{\frac{1+\tau \|\g_1\|^2}{1+\|\g_1\|^2}}-1\leq \max(1,\tau)-1\leq 0$. Thus, $\rho_t \le \rho_{t-1}$ for all $t\geq 1$.

\paragraph{Part (b-ii).} For all $t \ge 2$, we derive:
\beq
\ts \rho_{t-1}^2 - \rho_t^2
&\overset{\step{1}}{=}& \ts \frac{1+\frac{\tau}{t-1}\GGG_{t-1}}{1+\GGG_{t-1}} - \frac{1+\frac{\tau}{t}\GGG_t}{1+\GGG_t} ~~\overset{}{=}~~ \ts \big( \frac{1+\frac{\tau}{t-1}\GGG_{t-1}}{1+\GGG_{t-1}} - \frac{1+\frac{\tau}{t}\GGG_{t-1}}{1+\GGG_{t-1}} \big) + \big( \frac{1+\frac{\tau}{t}\GGG_{t-1}}{1+\GGG_{t-1}} - \frac{1+\frac{\tau}{t}\GGG_t}{1+\GGG_t} \big) \nn\\
&\overset{\step{2}}{=}& \ts \frac{\tau \GGG_{t-1}}{t(t-1)(1+\GGG_{t-1})} + \frac{(t-\tau)\|\g_t\|^2}{t(1+\GGG_{t-1})(1+\GGG_t)}~~\overset{\step{3}}{\le}~~
\ts \frac{\tau \hat{g}_t^2}{t(1+\GGG_{t-1})} + \frac{(t-\tau)\hat{g}_t^2}{t(1+\GGG_{t-1})(1+\GGG_t)} \nn \\
&\overset{}{=}& \ts \frac{\hat{g}_t^2}{1+\GGG_{t-1}} \left( \frac{\tau}{t} + \frac{t-\tau}{t(1+\GGG_t)} \right)
~~\overset{}{=} ~~\ts \frac{\hat{g}_t^2}{1+\GGG_{t-1}} \left( \frac{\tau(1+\GGG_t) + t-\tau}{t(1+\GGG_t)} \right)  ~~\overset{\step{4}}{=}~~ \ts \frac{\hat{g}_t^2}{1+\GGG_{t-1}} \rho_t^2, \nn
\eeq
\noi where step \step{1} uses the definition of $\rho_t^2$; step \step{2} uses $\frac{\tau}{t-1} - \frac{\tau}{t} = \frac{\tau}{t(t-1)}$, and $\GGG_t = \GGG_{t-1} + \|\g_t\|^2$; step \step{3} uses the average gradient bound $\frac{\GGG_{t-1}}{t-1} \le \hat{g}_{t-1}^2 \le \hat{g}_t^2$, and the current gradient bound $\|\g_t\|^2 \le \hat{g}_t^2$; step \step{4} uses the definition of $\rho_t$. Dividing both sides by $\rho_t^2$ and rearranging the terms yields $\tfrac{\rho_{t-1}^2}{\rho_t^2} \le 1 + \tfrac{\hat{g}_t^2}{1+\GGG_{t-1}}$. Taking the square root gives
\beq
\ts \tfrac{\rho_{t-1}}{\rho_t} \le \sqrt{1 + \tfrac{\hat{g}_t^2}{1+\GGG_{t-1}}}. \label{eq:bound:ratio}
\eeq
\noi Notably, Inequality (\ref{eq:bound:ratio}) holds trivially for $t=1$ since $\tfrac{\rho_0}{\rho_1}   = \sqrt{ (1+\|\g_1\|^2) / (1+\tau \|\g_1\|^2) } \leq \sqrt{1 + \|\g_1\|^2 } \leq \sqrt{1 + \hat{g}_1^2 } \leq \sqrt{1 + \hat{g}_t^2 / (1+\GGG_{0})}$.

\paragraph{Part (c).} We now prove $\rho_t^2 \le \frac{1+\hat g_t^2}{1+\GGG_t}$. For all $t\ge 2$, we derive
\beq
\ts \rho_t^2 ~~\overset{\step{1}}{=}~~ \tfrac{1+\frac{\tau}{t}\GGG_t}{1+\GGG_t} ~~\overset{\step{2}}{\le}~~ \tfrac{1+\frac{1}{t}\GGG_t}{1+\GGG_t}
~~\overset{\step{3}}{\le}~~ \tfrac{1+\hat g_t^2}{1+\GGG_t},\nn
\eeq
\noi where step \step{1} uses the definition of $\rho_t$; step \step{2} uses $\tau\in[0,1]$; step \step{3} uses $\frac{\GGG_t}{t}=\frac{1}{t}\sum_{i=1}^t \|g_i\|^2 \le \hat g_t^2$.

\paragraph{Part (d).} We now bound the variation term $\frac{(\rho_{t-1}-\rho_t)^2}{\rho_t}$ using the following inequalities:
\beq \label{eq:3322}
\ts \frac{(\rho_{t-1}-\rho_t)^2}{\rho_t} &\overset{}{=}& \ts \rho_t ( \frac{\rho_{t-1}}{\rho_t} - 1 )^2 ~~\overset{\step{1}}{\le}~~
\ts \frac{ ( \frac{\rho_{t-1}}{\rho_t} - 1 )^2}{ \rho_{t-1}/\rho_t} ~~\overset{\step{2}}{\le}~~\frac{ ( \sqrt{1 + {\hat{g}_t^2}/{(1+\GGG_{t-1})}} - 1 )^2}{ \sqrt{1 + {\hat{g}_t^2}/{(1+\GGG_{t-1})}} } \nn\\
&\overset{\step{3}}{\leq}& \ts \frac{ ( {\hat{g}_t^2} / ({1+\GGG_{t-1}} ))^{3/2}}{1 + {\hat{g}_t^2}/(1+\GGG_{t-1})} ~~\overset{}{=}~~
\ts \frac{\hat{g}_t^3}{\sqrt{1+\GGG_{t-1}} \left( 1+\GGG_{t-1}+\hat{g}_t^2 \right)}
~~\overset{\step{4}}{\le}~~
\ts \frac{\hat{g}_t^3}{1+\GGG_{t-1}+\|\g_t\|^2} ~~\overset{\step{5}}{=}~~
\ts \frac{\hat{g}_t^3}{1+\GGG_t},  
\eeq
\noi where step \step{1} uses $\rho_{t-1} \le 1$; step \step{2} uses the fact that the function $h(x)=\tfrac{(x-1)^2}{x}$ is an increasing function for all $x\geq 1$ with $x=\tfrac{\rho_{t-1}}{\rho_t}$, and the upper bound of $\tfrac{\rho_{t-1}}{\rho_t}$ in Inequality (\ref{eq:bound:ratio}); step \step{3} uses $\tfrac{(\sqrt{y+1}-1)^2}{\sqrt{1+y}} \leq \tfrac{y^2}{\sqrt{1+y} (\sqrt{1+y}+1)^2} \leq \tfrac{y^2}{\sqrt{1+y} \sqrt{ y(1+y) }} \leq \tfrac{y^{3/2}}{1+y}$ for all $y=\frac{\hat{g}_t^2}{1+\GGG_{t-1}}$; step \step{4} uses $\hat{g}_t^2 \ge \|\g_t\|^2$; step \step{5} uses $\GGG_t = \GGG_{t-1} + \|\g_t\|^2$. Additionally, Inequality (\ref{eq:3322}) holds for $t=1$ since $\frac{(\rho_0-\rho_1)^2}{\rho_1}
\le
\frac{\left(1-(1+u^2)^{-1/2}\right)^2}{(1+u^2)^{-1/2}}
\le
\frac{u^3}{1+u^2}
=
\frac{\hat g_1^3}{1+\|\g_1\|^2}$, where $u=\|\g_1\|$.

\end{proof}

\section{Proofs for Section \ref{sect:it:EMA:M}}
\label{app:sect:it:EMA:M}

\subsection{Proof of Lemma \ref{lemma:bound:m:using:g:1}}
\label{app:lemma:bound:m:using:g:1}

\begin{proof} Let $\GGG_t := \sum_{i=1}^t \|\g_i\|^2$.

\paragraph{Bounding the term $\sum_{t=1}^T \|\m_t\|^2$.} By the convexity of the squared norm $\|\cdot\|^2$, the EMA update $\m_t = (1-\alpha_t) \m_{t-1} + \alpha_t \g_t$ implies $\|\m_t\|^2 \leq (1-\alpha_t) \|\m_{t-1}\|^2 + \alpha_t \|\g_t\|^2$. Summing this recursion over $t=1,\dots,T$ with $\m_0 = \mathbf{0}$ yields
\beq
&& \ts \sum_{t=1}^T \alpha_t \|\g_t\|^2 ~~\geq ~~ \ts \big(\sum_{t=1}^T \|\m_t\|^2\big) - \big( \sum_{t=1}^{T} (1-\alpha_{t}) \|\m_{t-1}\|^2 \big) \nn\\
&\overset{\step{1}}{=}&\ts \|\m_T\|^2 +  \sum_{t=1}^{T-1} \alpha_{t+1} \|\m_{t}\|^2 ~~\overset{\step{2}}{\geq} ~~\ts \|\m_T\|^2 +  \alpha_T \sum_{t=1}^{T-1} \|\m_{t}\|^2 ~~\geq ~~ \alpha_T \sum_{t=1}^T \|\m_t\|^2,\nn
\eeq
where step \step{1} uses $\m_0=\zero$; step \step{2} uses the monotonic non-increasing property of $\alpha_t$ (i.e., $\alpha_{t+1} \le \alpha_t$) and $\alpha_T \in (0,1]$. This leads to: $\sum_{t=1}^T \|\m_t\|^2 \le \frac{1}{\alpha_T} \sum_{t=1}^T \alpha_t \|\g_t\|^2$.

\paragraph{Bounding the term $\sum_{t=1}^T \alpha_t \|\g_t\|^2$.} We further derive:
\beq
&&\ts \sum_{t=1}^T \alpha_t \|\g_t\|^2 ~~\overset{\step{1}}{=}~~ \ts \sum_{t=1}^T \sqrt{\frac{1+\frac{\tau}{t}\GGG_t}{1+\GGG_t}} \|\g_t\|^2 ~~\overset{\step{2}}{\leq}~~\sqrt{1+\tau {\hG}^2} \sum_{t=1}^T \frac{\|\g_t\|^2}{\sqrt{1+\GGG_t}} \nn\\
&\overset{\step{3}}{\leq}& \ts \sqrt{1+\tau {\hG}^2} \sum_{t=1}^T 2(\sqrt{1+\GGG_t}-\sqrt{1+\GGG_{t-1}}) ~~= ~~2\sqrt{1+\tau {\hG}^2} (\sqrt{1+\GGG_T}-1).\nn
\eeq
\noi where step \step{1} uses the definition of $\alpha_t$; step \step{2} uses $1+\frac{\tau}{t}\GGG_t \le 1+\tau {\hG}^2$; step \step{3} uses the standard integral bound $\frac{\|\g_t\|^2}{\sqrt{1+\GGG_t}} \le \int_{\GGG_{t-1}}^{\GGG_t} \frac{dx}{\sqrt{1+x}} = 2(\sqrt{1+\GGG_t}-\sqrt{1+\GGG_{t-1}})$.

\paragraph{Final Result.} Combining the above bounds, we obtain
\beq
&& \ts \sum_{t=1}^T \|\m_t\|^2 ~~\le~~ \frac{1}{\alpha_T} \sum_{t=1}^T \alpha_t \|\g_t\|^2 ~~\overset{\step{1}}{\leq}~~ 2\sqrt{1+\tau {\hG}^2} \sqrt{1+\GGG_T}(\sqrt{1+\GGG_T}-1) \nn\\
&=& 2\sqrt{1+\tau {\hG}^2} (1+\GGG_T - \sqrt{1+\GGG_T}) ~~\overset{\step{2}}{\leq}~~ 2(1+\sqrt{\tau}{\hG}) \GGG_T, \nn
\eeq
where step \step{1} uses $\frac{1}{\alpha_T} = \sqrt{\frac{1+\GGG_T}{1+\frac{\tau}{T}\GGG_T}} \le \sqrt{1+\GGG_T}$ since $1+\frac{\tau}{T}\GGG_T \ge 1$; step \step{2} uses $\sqrt{1+\GGG_T} \ge 1$ and $\sqrt{1+\tau {\hG}^2}  \le 1 + \sqrt{\tau}{\hG}$.

\end{proof}

 \subsection{Proof of Lemma \ref{lemma:bound:EMA:M:v}}
\label{app:lemma:bound:EMA:M:v}
\begin{proof} Since $\mathbf{v}_{t} = (1-\beta_t) \mathbf{v}_{t-1} + \beta_t \mathbf{g}_t \odot \mathbf{g}_t$, and assuming $\beta_t \in [0, 1]$, using the convexity of $\lVert \cdot \rVert$ and the inequality $\lVert \mathbf{x} \odot \mathbf{x} \rVert \leq \lVert \mathbf{x} \rVert^2$, we have:
    \begin{align*}
        \lVert \mathbf{v}_t \rVert ~\leq~ (1-\beta_t) \lVert \mathbf{v}_{t-1} \rVert + \beta_t \lVert \mathbf{g}_t \odot \mathbf{g}_t \rVert ~\leq~ (1-\beta_t) \lVert \mathbf{v}_{t-1} \rVert + \beta_t \lVert \mathbf{g}_t \rVert^2 ~\leq~\max \left( \lVert \mathbf{v}_{t-1} \rVert, \lVert \mathbf{g}_t \rVert^2 \right).
    \end{align*}

\paragraph{Bounding the term $\|\v_t\|$.} Using this inequality recursively, and assuming $\mathbf{v}_0 = \mathbf{0}$, we obtain:
    \begin{align*}
        \lVert \mathbf{v}_t \rVert ~\leq~ \max \left( \lVert \mathbf{v}_0 \rVert, \max_{i=1}^t \lVert \mathbf{g}_i \rVert^2 \right) ~=~ \max_{i=1}^t \lVert \mathbf{g}_i \rVert^2 ~\leq~ {\hG}^2.
    \end{align*}


\paragraph{Bounding the term $\|\s_t\|$.} Since $\m_t$ is a convex combination of $\m_{t-1}$ and $\g_t$, Assumption \ref{ass:BG} gives $\|\m_t\|\le {\hG}$. Using the fact that $\|\nabla f(\x_t)\|\le {\hG}$, we have almost surely:
\[
\|\s_t\| =\|\m_t-\nabla f(\x_t)\| \le 2 {\hG}.
\]

\end{proof}

\subsection{Proof of Lemma \ref{lemma:bound:eee:1}}
\label{app:lemma:bound:eee:1}

\begin{proof}
We prove the result for OptEMA-M, where $\beta_t=\beta$ is fixed and
$\alpha_t=\rho_t$. By Lemma \ref{lemma:corrected}, the sequence
$\{\alpha_t\}_{t\geq1}$ is non-increasing. Throughout the proof, we use the
conventions $\alpha_0=\gamma_0=1$.

\paragraph{Definitions and measurability.}
Define $\s_t:=\m_t-\nabla f(\x_t)$, $\epsilons_t:=\nabla f(\x_t;\xi_t)-\nabla f(\x_t)$.

\noi Let $A_t:=\s_{t-1}+\nabla f(\x_{t-1})-\nabla f(\x_t)$, $\z_t:=(1-\alpha_t)A_t$, $\hat{\z}_t:=(1-\alpha_{t-1})A_t$.
 
\noi Set $\Delta_t:=\alpha_{t-1}-\alpha_t\geq0$ and $\lambda:=\frac{\theta^2L^2}{\varepsilon^2}$.

\noi We use the natural filtration $\FF_t:=\sigma(\xi_1,\ldots,\xi_t)$. Then $\x_t,\m_{t-1},\v_{t-1},\alpha_{t-1},\gamma_{t-1}$, $A_t$,
$\hat{\z}_t$, and $\s_{t-1}$ are $\FF_{t-1}$-measurable, while
$\g_t$, $\alpha_t=\rho_t$, $\z_t$, $\m_t$, and $\s_t$ are
$\FF_t$-measurable. In particular, $\alpha_t$ is generally not
$\FF_{t-1}$-measurable. This is the only measurability issue that must be
handled carefully below. Throughout the proof, we keep every factor involving
$\alpha_t$ inside the conditional or unconditional expectation; in particular,
we never use an identity of the form
$\EEE[\alpha_t Z_t\mid \FF_{t-1}]
=\alpha_t\EEE[Z_t\mid \FF_{t-1}]$.

\paragraph{Bounding the gradient-difference term.}
Recall that, for all $t\ge1$, $\x_t=\x_{t-1}
-\theta\gamma_{t-1}\frac{\m_{t-1}}{\varepsilon+\sqrt{\v_{t-1}}}$, where the case $t=1$ follows from the initialization
$\x_0=\x_1$, $\m_0=\zero$, and $\v_0=\zero$. By the $L$-smoothness of $f$,
we have
\beq \label{eq:bound:gradient:diff}
\ts \|\nabla f(\x_t)-\nabla f(\x_{t-1})\|^2 &\overset{\step{1}}{\leq}& \ts L^2\|\x_t-\x_{t-1}\|^2 \,\,\overset{\step{2}}{=}\,\, L^2\theta^2\gamma_{t-1}^2
\left\| \frac{\m_{t-1}}{\varepsilon+\sqrt{\v_{t-1}}}
\right\|^2 \nonumber\\
&\overset{\step{3}}{\leq}& \ts \frac{L^2\theta^2}{\varepsilon^2}
\gamma_{t-1}^2\|\m_{t-1}\|^2\,\,=\,\, \lambda\gamma_{t-1}^2\|\m_{t-1}\|^2,
\eeq
where step \step{1} uses the $L$-Lipschitz continuity of $\nabla f$; step \step{2} uses the update identity above; and step \step{3} uses
$\varepsilon+\sqrt{\v_{t-1}}\geq\varepsilon$ componentwise.

\paragraph{Recursive representation of $\s_t$.}
Using the momentum update $\m_t=(1-\alpha_t)\m_{t-1}+\alpha_t\nabla f(\x_t;\xi_t)$, we obtain
\beq \label{eq:bound:s}
\s_t
&\overset{\step{1}}{=}& \ts (1-\alpha_t)\m_{t-1}
+\alpha_t\nabla f(\x_t;\xi_t) -\nabla f(\x_t) \nonumber\\
&=& \ts (1-\alpha_t)\big(\m_{t-1}-\nabla f(\x_{t-1})\big)
+(1-\alpha_t)\big(\nabla f(\x_{t-1})-\nabla f(\x_t)\big) + \alpha_t  \epsilons_t \nonumber\\
&\overset{\step{2}}{=}& (1-\alpha_t) \big(\s_{t-1}+\nabla f(\x_{t-1})-\nabla f(\x_t)\big)+\alpha_t\epsilons_t \nonumber\\
&=& \ts \z_t+\alpha_t\epsilons_t,
\eeq
\noi where step \step{1} uses the definition of $\s_t$, and step \step{2} uses $\s_{t-1}=\m_{t-1}-\nabla f(\x_{t-1})$.

\paragraph{Bounding $\|\z_t\|^2$.}
For every sample path, by Young's inequality, we have
\beq \label{eq:z:squared:new}
\|\z_t\|^2
&=&
(1-\alpha_t)^2
\left\|
\s_{t-1}+\nabla f(\x_{t-1})-\nabla f(\x_t)
\right\|^2
\nonumber\\
&\overset{\step{1}}{\leq}&
(1+\alpha_t)(1-\alpha_t)^2\|\s_{t-1}\|^2
+
\left(1+\frac1{\alpha_t}\right)(1-\alpha_t)^2
\|\nabla f(\x_{t-1})-\nabla f(\x_t)\|^2
\nonumber\\
&\overset{\step{2}}{\leq}&
(1-\alpha_t)\|\s_{t-1}\|^2
+
\frac1{\alpha_t}
\|\nabla f(\x_{t-1})-\nabla f(\x_t)\|^2
\nonumber\\
&\overset{\step{3}}{\leq}&
(1-\alpha_t)\|\s_{t-1}\|^2
+
\frac{\lambda}{\alpha_t}
\gamma_{t-1}^2\|\m_{t-1}\|^2 .
\eeq
Here step \step{1} uses
$\|\a+\b\|^2\leq(1+\eta)\|\a\|^2+(1+\eta^{-1})\|\b\|^2$
with $\eta=\alpha_t$; step \step{2} uses
$(1+\alpha_t)(1-\alpha_t)\leq1$, $\left(1+\frac1{\alpha_t}\right)(1-\alpha_t)^2 \leq
\frac1{\alpha_t}$; and step \step{3} uses \eqref{eq:bound:gradient:diff}.

\paragraph{Bounding $\|\z_t-\hat{\z}_t\|^2$.}
Since $\z_t-\hat{\z}_t = (\alpha_{t-1}-\alpha_t)A_t = \Delta_t A_t$, we have
\beq \label{eq:lemma:bound:eee:1:surrogate}
\|\z_t-\hat{\z}_t\|^2
&=&
\Delta_t^2
\left\|
\s_{t-1}+\nabla f(\x_{t-1})-\nabla f(\x_t)
\right\|^2
\nonumber\\
&\overset{\step{1}}{\leq}&
2\Delta_t^2\|\s_{t-1}\|^2
+
2\Delta_t^2
\|\nabla f(\x_{t-1})-\nabla f(\x_t)\|^2
\nonumber\\
&\overset{\step{2}}{\leq}&
2\Delta_t^2\|\s_{t-1}\|^2
+
2\Delta_t^2\lambda
\gamma_{t-1}^2\|\m_{t-1}\|^2 .
\eeq
Here step \step{1} uses $\|\a+\b\|^2\leq2\|\a\|^2+2\|\b\|^2$, and
step \step{2} uses \eqref{eq:bound:gradient:diff}.

\paragraph{Handling the non-predictable cross term.}
From \eqref{eq:bound:s}, we have
\[
\|\s_t\|^2
=
\|\z_t\|^2
+
2\langle \z_t,\alpha_t\epsilons_t\rangle
+
\alpha_t^2\|\epsilons_t\|^2 .
\]
Since $\alpha_t$ is not $\FF_{t-1}$-measurable, we cannot use
$\EEE[\alpha_t\epsilons_t\mid\FF_{t-1}]=0$. Instead, we split the cross term as
\beq \label{eq:cross:split:new}
2\langle \z_t,\alpha_t\epsilons_t\rangle
&=&
2\langle \z_t-\hat{\z}_t,\alpha_{t-1}\epsilons_t\rangle
+
2\langle \z_t,(\alpha_t-\alpha_{t-1})\epsilons_t\rangle
+ 2\langle \hat{\z}_t,\alpha_{t-1}\epsilons_t\rangle .
\eeq
The last term is the only term handled by conditional unbiasedness. Since
$\hat{\z}_t$ and $\alpha_{t-1}$ are $\FF_{t-1}$-measurable, we have
\[
\EEE\!\left[
\left.
\langle \hat{\z}_t,\alpha_{t-1}\epsilons_t\rangle
\right|
\FF_{t-1}
\right]
=
\left\langle
\hat{\z}_t,\alpha_{t-1}
\EEE[\epsilons_t\mid\FF_{t-1}]
\right\rangle
=0.
\]

\paragraph{Pathwise control of the remaining cross terms.}
We first bound
$2\langle \z_t-\hat{\z}_t,\alpha_{t-1}\epsilons_t\rangle$.
If $\Delta_t>0$, define
\[
\eta_t:=\frac{8\Delta_t^2}{\alpha_t}>0 .
\]
By Young's inequality and \eqref{eq:lemma:bound:eee:1:surrogate},
\beq \label{eq:young:first:positive:new}
2\langle \z_t-\hat{\z}_t,\alpha_{t-1}\epsilons_t\rangle
&\overset{\step{1}}{\leq}&
\frac1{\eta_t}\|\z_t-\hat{\z}_t\|^2
+
\eta_t\alpha_{t-1}^2\|\epsilons_t\|^2
\nonumber\\
&\overset{\step{2}}{\leq}&
\frac{2\Delta_t^2}{\eta_t}\|\s_{t-1}\|^2
+
\frac{2\Delta_t^2\lambda}{\eta_t}
\gamma_{t-1}^2\|\m_{t-1}\|^2
+
\eta_t\alpha_{t-1}^2\|\epsilons_t\|^2
\nonumber\\
&=&
\frac{\alpha_t}{4}\|\s_{t-1}\|^2
+
\frac{\alpha_t\lambda}{4}
\gamma_{t-1}^2\|\m_{t-1}\|^2
+
\frac{8\Delta_t^2}{\alpha_t}\alpha_{t-1}^2
\|\epsilons_t\|^2 .
\eeq
If $\Delta_t=0$, then $\z_t-\hat{\z}_t=0$, and the left-hand side of
\eqref{eq:young:first:positive:new} is exactly zero. In this case, the
right-hand side of \eqref{eq:young:first:positive:new} is nonnegative and all
terms containing $\Delta_t^2/\alpha_t$ vanish. Therefore
\eqref{eq:young:first:positive:new} holds for all sample paths, without ever
dividing by zero.

Similarly, using $\alpha_t-\alpha_{t-1}=-\Delta_t$, we have
\beq \label{eq:young:second:new}
2\langle \z_t,(\alpha_t-\alpha_{t-1})\epsilons_t\rangle
&\leq&
2\Delta_t\|\z_t\|\|\epsilons_t\|
\nonumber\\
&\overset{\step{1}}{\leq}&
\frac{\alpha_t}{2}\|\z_t\|^2
+
\frac{2\Delta_t^2}{\alpha_t}\|\epsilons_t\|^2 .
\eeq

\paragraph{Conditional recursion before using the self-normalized bound.}
Combining \eqref{eq:cross:split:new}, \eqref{eq:young:first:positive:new},
and \eqref{eq:young:second:new}, and then taking conditional expectation with
respect to $\FF_{t-1}$, gives
\beq \label{eq:lemma:bound:eee:1:expand:new}
\EEE[\|\s_t\|^2\mid\FF_{t-1}]
&\leq&
\EEE\!\left[
\left.
\left(1+\frac{\alpha_t}{2}\right)\|\z_t\|^2
+
\frac{\alpha_t}{4}\|\s_{t-1}\|^2
+
\frac{\alpha_t\lambda}{4}
\gamma_{t-1}^2\|\m_{t-1}\|^2
\right|
\FF_{t-1}
\right]
\nonumber\\
&&
+
\EEE\!\left[
\left.
\left(
\alpha_t^2
+
\frac{8\Delta_t^2}{\alpha_t}\alpha_{t-1}^2
+
\frac{2\Delta_t^2}{\alpha_t}
\right)
\|\epsilons_t\|^2
\right|
\FF_{t-1}
\right]
\nonumber\\
&\overset{\step{1}}{\leq}&
\EEE\!\left[
\left.
\left(1+\frac{\alpha_t}{2}\right)\|\z_t\|^2
+
\frac{\alpha_t}{4}\|\s_{t-1}\|^2
+
\frac{\alpha_t\lambda}{4}
\gamma_{t-1}^2\|\m_{t-1}\|^2
\right|
\FF_{t-1}
\right]
\nonumber\\
&&
+
\EEE\!\left[
\left.
\left(
\alpha_t^2
+
\frac{10\Delta_t^2}{\alpha_t}
\right)
\|\epsilons_t\|^2
\right|
\FF_{t-1}
\right],
\eeq
where step \step{1} uses $\alpha_{t-1}\leq1$.

Next, using \eqref{eq:z:squared:new}, we obtain pathwise
\beq \label{eq:z-plus-s:new}
\left(1+\frac{\alpha_t}{2}\right)\|\z_t\|^2
+
\frac{\alpha_t}{4}\|\s_{t-1}\|^2
&\overset{\step{1}}{\leq}&
\left(1+\frac{\alpha_t}{2}\right)(1-\alpha_t)\|\s_{t-1}\|^2
+
\frac{\alpha_t}{4}\|\s_{t-1}\|^2
\nonumber\\
&&
+
\left(1+\frac{\alpha_t}{2}\right)
\frac{\lambda}{\alpha_t}
\gamma_{t-1}^2\|\m_{t-1}\|^2
\nonumber\\
&\overset{\step{2}}{\leq}&
\left(1-\frac{\alpha_t}{4}\right)\|\s_{t-1}\|^2
+
\frac{3\lambda}{2\alpha_t}
\gamma_{t-1}^2\|\m_{t-1}\|^2 .
\eeq
Moreover,
\[
\frac{\alpha_t\lambda}{4}
\gamma_{t-1}^2\|\m_{t-1}\|^2
\leq
\frac{\lambda}{4\alpha_t}
\gamma_{t-1}^2\|\m_{t-1}\|^2 .
\]
Substituting these estimates into
\eqref{eq:lemma:bound:eee:1:expand:new} yields
\beq \label{eq:before:rho:bounds:new}
\EEE[\|\s_t\|^2\mid\FF_{t-1}]
&\leq&
\EEE\!\left[
\left.
\left(1-\frac{\alpha_t}{4}\right)\|\s_{t-1}\|^2
+
\frac{3\lambda}{\alpha_t}
\gamma_{t-1}^2\|\m_{t-1}\|^2
\right|
\FF_{t-1}
\right]
\nonumber\\
&&
+
\EEE\!\left[
\left.
\left(
\alpha_t^2
+
\frac{10\Delta_t^2}{\alpha_t}
\right)
\|\epsilons_t\|^2
\right|
\FF_{t-1}
\right].
\eeq

\paragraph{Using the corrected AdaGrad-Norm bounds.}
Since $\alpha_t=\rho_t$, Lemma \ref{lemma:corrected} gives the
desired bounds for all $t\ge1$. By Assumption \ref{ass:BG}, $\hat g_t\leq {\hG}$ almost surely. Hence
\beq \label{eq:alpha-delta-bound:new}
\alpha_t^2+\frac{10\Delta_t^2}{\alpha_t}
&\leq&
\frac{1+{\hG}^2+10 {\hG}^3}{1+\sum_{i=1}^t\|\g_i\|^2}.
\eeq
Moreover, the ratio bound implies
\beq \label{eq:alpha-ratio-bound:new}
\frac1{\alpha_t}
&\leq&
\frac{1+{\hG}}{\alpha_{t-1}} .
\eeq
Substituting \eqref{eq:alpha-delta-bound:new} and
\eqref{eq:alpha-ratio-bound:new} into \eqref{eq:before:rho:bounds:new}, we get
\beq \label{eq:before:self:normalized:new}
\EEE[\|\s_t\|^2\mid\FF_{t-1}]
&\leq&
\EEE\!\left[
\left.
\left(1-\frac{\alpha_t}{4}\right)\|\s_{t-1}\|^2
+
\frac{3(1+{\hG})\lambda}{\alpha_{t-1}}
\gamma_{t-1}^2\|\m_{t-1}\|^2
\right|
\FF_{t-1}
\right]
\nonumber\\
&&
+
(1+{\hG}^2+10 {\hG}^3)
\EEE\!\left[
\left.
\frac{\|\epsilons_t\|^2}
{1+\sum_{i=1}^t\|\g_i\|^2}
\right|
\FF_{t-1}
\right].
\eeq

\paragraph{Self-normalized noise domination.}
By Lemma \ref{lemma:self:normalized:noise}, we have
\[
\EEE\!\left[
\left.
\frac{\|\epsilons_t\|^2}
{1+\sum_{i=1}^t\|\g_i\|^2}
\right|
\FF_{t-1}
\right]
\leq
(4+2 {\hG}^2)
\EEE\!\left[
\left.
\frac{\|\g_t\|^2}
{1+\sum_{i=1}^t\|\g_i\|^2}
\right|
\FF_{t-1}
\right].
\]
Recall that
\[
\lambda_1:=3(1+{\hG})\lambda
=
\frac{3\theta^2L^2}{\varepsilon^2}(1+{\hG}),
\qquad
\lambda_2:=(1+{\hG}^2+10 {\hG}^3)(4+2 {\hG}^2).
\]
Combining the last two displays gives
\beq \label{eq:conditional:tracking:new}
\EEE[\|\s_t\|^2\mid\FF_{t-1}]
&\leq&
\EEE\!\left[
\left.
\left(1-\frac{\alpha_t}{4}\right)\|\s_{t-1}\|^2
+
B_t
\right|
\FF_{t-1}
\right],
\eeq
where
\[
B_t:=
\lambda_1
\frac{\gamma_{t-1}^2}{\alpha_{t-1}}
\|\m_{t-1}\|^2
+
\lambda_2
\frac{\|\g_t\|^2}
{1+\sum_{i=1}^t\|\g_i\|^2}.
\]

\paragraph{Predictable multiplier.}
Let $H_{t-1}\geq0$ be any integrable $\FF_{t-1}$-measurable random variable.
Multiplying \eqref{eq:conditional:tracking:new} by $H_{t-1}$ and then taking
total expectation gives
\beq
\EEE\!\left[H_{t-1}\|\s_t\|^2\right]
&=&
\EEE\!\left[
H_{t-1}\EEE[\|\s_t\|^2\mid\FF_{t-1}]
\right]
\nonumber\\
&\overset{\step{1}}{\leq}&
\EEE\!\left[
H_{t-1}
\EEE\!\left[
\left.
\left(1-\frac{\alpha_t}{4}\right)\|\s_{t-1}\|^2+B_t
\right|
\FF_{t-1}
\right]
\right]
\nonumber\\
&\overset{\step{2}}{=}&
\EEE\!\left[
H_{t-1}
\left(
\left(1-\frac{\alpha_t}{4}\right)\|\s_{t-1}\|^2+B_t
\right)
\right].
\eeq
Here step \step{1} uses \eqref{eq:conditional:tracking:new}, and step
\step{2} uses the $\FF_{t-1}$-measurability of $H_{t-1}$ together with the
tower property. 
\end{proof}

\subsection{Proof of Lemma \ref{lemma:OptEMA:M:suff:dec}}
\label{app:lemma:OptEMA:M:suff:dec}

\begin{proof} Define $\mathbf{s}_t := \mathbf{m}_{t} - \nabla f(\mathbf{x}_t)$, and $h_t := \varepsilon + \sqrt{\|\mathbf{v}_t\|_{\infty}}$.

\noi Define $c_0:=\tfrac{\theta}{ 2(1+{\hG})}$, $c_1:=\tfrac{\theta}{ 2\varepsilon^2} (1+{\hG})$, and $c_2:=\tfrac{L\theta^2}{2\varepsilon^2}$. 

\noi By the \(L\)-smoothness of $f(\cdot)$, we have, almost surely,
\beq \label{eq:L:smooth}
&&\ts  f(\x_{t+1}) - f(\x_t)  \nn\\
&\leq & \ts \tfrac{L}{2}\|\x_{t+1}-\x_t\|_2^2 + \la\nabla f(\x_t),\x_{t+1}-\x_t\ra   \nn \\
&\overset{\step{1}}{=}&   \tfrac{L}{2}\theta^2\gamma_t^2\|\tfrac{\m_{t}}{\varepsilon+ \sqrt{\v_{t}}}\|_2^2 + \theta \gamma_t \left( - \la\nabla f(\x_t), \tfrac{\nabla f(\x_t)}{\varepsilon+ \sqrt{ \v_{t}}}\ra + \la\nabla f(\x_t),  \tfrac{\nabla f(\x_t) - \m_{t} }{\varepsilon+ \sqrt{\v_{t}}} \ra \right)   \nn \\
&\overset{\step{2}}{\leq}&   \tfrac{L\theta^2\gamma_t^2}{2 \varepsilon^2} \|\m_{t}\|_2^2 + \theta \gamma_t \left( -\tfrac{1}{ h_t }\|\nabla f(\x_t)\|_2^2 + \tfrac{1}{\varepsilon}\| \nabla f(\x_t)\| \cdot \| \m_t - \nabla f(\x_t) \|\right)   \nn \\
&\overset{}{=}&  \tfrac{L \theta^2\gamma_t^2}{2 \varepsilon^2} \|\m_{t}\|_2^2 + \theta \gamma_t \left( -  \tfrac{ 1 }{ 2 h_t }\|\nabla f(\x_t)\|_2^2 -  \tfrac{ 1 }{ 2 h_t }\|\nabla f(\x_t)\|_2^2 + \tfrac{1}{\varepsilon} \|\nabla f(\x_t)\|\cdot \|\s_t\| \right)    \nn \\
&\overset{\step{3}}{\leq}&  \tfrac{L\theta^2\gamma_t^2}{2\varepsilon^2} \|\m_{t}\|_2^2 + \theta \gamma_t \left( -  \tfrac{1 }{ 2 h_t }\|\nabla f(\x_t)\|_2^2 +  \tfrac{ h_t }{ 2\varepsilon^2} \| \s_t \|^2   \right)    \nn \\
&\overset{\step{4}}{\leq}&  \tfrac{L\theta^2\gamma_t^2}{2\varepsilon^2} \|\m_{t}\|_2^2 + \theta \gamma_t \left( -  \tfrac{ 1 }{ 2 (1+{\hG}) }\|\nabla f(\x_t)\|_2^2 +  \tfrac{ 1+ {\hG}  }{ 2 \varepsilon^2} \| \s_t \|^2   \right)   , \nn
\eeq
\noi where step \step{1} uses $\x_{t+1} = \x_t - \theta \gamma_t \cdot \tfrac{\m_{t}}{ \varepsilon + \sqrt{ \v_{t}}}$; step \step{2} uses Cauchy-Schwarz inequality, $\mathbf{v}_t \geq 0$, and the definition of $h_t$; step \step{3} uses the inequality $-\tau x^2 + ax \leq \tfrac{a^2}{4\tau}$ with $\tau=\tfrac{1}{2h_t}$, $a=\tfrac{1}{\varepsilon}\|\s_t\|$ and $x=\|\nabla f(\x_t)\|$; step \step{4} uses the fact that $h_t:=\varepsilon + \sqrt{\max(\v_t)} \leq \varepsilon + \sqrt{ \|\v_t\|} \leq 1+{\hG}$, which is implied by $\varepsilon\in(0,1)$ and Lemma \ref{lemma:bound:EMA:M:v}.

Additionally, the following inequalities hold almost surely,
\beq
&& \ts f(\x_{t+1}) - f(\x_{t}) \nn\\
&\leq & \ts \tfrac{L}{2}\|\x_{t+1}-\x_t\|_2^2 + \la\nabla f(\x_t),\x_{t+1}-\x_t\ra   \nn \\
&\le& \tfrac{L\theta^2\gamma_t^2}{2\varepsilon^2}\|\m_t\|^2 - \theta\gamma_t \la \nabla f(\x_t),
\frac{\m_t}{\varepsilon+\sqrt{\v_t}}\ra    \nn   \\
&=& \frac{L\theta^2}{2\varepsilon^2}\gamma_t^2\|\m_t\|^2   - \theta\gamma_t \la \m_t,\frac{\m_t}{\varepsilon+\sqrt{\v_t}}\ra + \theta\gamma_t\la\s_t,\frac{\m_t}{\varepsilon+\sqrt{\v_t}}\ra    \nn \\
&\le& \frac{L\theta^2}{2\varepsilon^2}\gamma_t^2\|\m_t\|^2 -\frac{\theta}{1+{\hG}}\gamma_t\|\m_t\|^2
+\frac{\theta}{\varepsilon}\gamma_t\|\s_t\|\|\m_t\| \nn       \\
&\le& c_2\gamma_t^2\|\m_t\|^2 - c_0\gamma_t\|\m_t\|^2 + c_1\gamma_t\|\s_t\|^2.\nn
\eeq

\end{proof}

\subsection{Proof of Lemma \ref{lemma:sum:tail}}
\label{app:lemma:sum:tail}

\begin{proof} Define $\hat{g}_t = \max_{1\leq i\leq t} \|\g_i\|$ and $\alpha_t=  \sqrt{\frac{1 + \frac{\tau}{t}\sum_{i=1}^t \| \g_i \|^2}{1 + \sum_{i=1}^t \| \g_i \|^2}}$, where $\tau\in[0,1]$.

\noi Define $\gamma_t = \min \left(
\alpha_t,\,
\sqrt{\alpha_t}\left( 1 + \sum_{j=1}^t \|\m_j\|^2 \right)^{-1/2}
\right)$.

\noi Define $B_t:=\lambda_1 \tfrac{\gamma_{t-1}^2}{\alpha_{t-1}}  \| \m_{t-1}\|^2    +  \lambda_2 \tfrac{ \|\g_t\|^2 }{ 1 + \sum_{i=1}^t \|\g_i\|^2 }$.

\noi Define $C_b:=(\lambda_1+\lambda_2) \cdot (1+\ln \kappa  ) \cdot \ln(e+ \hG^2)$, and $C_c:=4 (1+{\hG})\cdot {\hG}^2 + 4 (1+{\hG}) C_b$.

\paragraph{Bounding the term $\ln (e + \GGG_T )$.} We derive:
\beq \label{eq:log:GT}
\ln(e+\GGG_T) \,\,\xixi{1}{\leq}\,\,
\ln(e+{\hG}^2T)
\,\,\xixi{2}{\leq}\,\,
\ln(e+{\hG}^2)\ln(e+T),
\eeq
\noi where step \step{1} uses Assumption \ref{ass:BG}, leading to $\GGG_T\le {\hG}^2T$ almost surely; step \step{2} uses $\ln(e+ab) \leq \ln(e+a)\cdot \ln(e+b)$.

\paragraph{Bounding the term $\sum_{t=1}^T  B_t$.} We derive:
\beq \label{eq:sum:B}
\ts \sum_{t=1}^T  B_t &\overset{\step{1}}{=}& \ts  \left( \lambda_1 \sum_{t=1}^T \tfrac{\gamma_{t-1}^2}{\alpha_{t-1}}  \| \m_{t-1}\|^2  \right)  +  \left(\lambda_2 \sum_{t=1}^T \tfrac{ \|\g_t\|^2 }{ 1 + \sum_{i=1}^t \|\g_i\|^2 } \right) \nn\\
&\overset{\step{2}}{\leq}& \ts \left(  \lambda_1  \sum_{t=1}^T \tfrac{\gamma_{t}^2}{\alpha_{t}}  \| \m_{t}\|^2  \right)  +  \left( \lambda_2  \sum_{t=1}^T \tfrac{ \|\g_t\|^2 }{ 1 + \sum_{i=1}^t \|\g_i\|^2 } \right) \nn\\
&\overset{\step{3}}{\leq}& \ts \left(  \lambda_1  \sum_{t=1}^T  \tfrac{\| \m_{t}\|^2}{ 1 + \sum_{i=1}^t \|\m_i\|^2 }   \right)  +  \left( \lambda_2  \sum_{t=1}^T \tfrac{ \|\g_t\|^2 }{ 1 + \sum_{i=1}^t \|\g_i\|^2 } \right) \nn\\
&\overset{\step{4}}{\leq}& \ts   \lambda_1  \ln \left( 1 + \sum_{i=1}^T \|\m_i\|^2   \right)  +   \lambda_2 \ln \left( 1 + \sum_{i=1}^T \|\g_i\|^2   \right) \nn\\
&\overset{\step{5}}{\leq}& \ts   \lambda_1  \ln \left( 1 + \kappa \sum_{i=1}^T \|\g_i\|^2   \right)  +   \lambda_2 \ln \left( 1 + \sum_{i=1}^T \|\g_i\|^2   \right) \nn\\
&\overset{\step{6}}{\leq}& \ts   (\lambda_1+\lambda_2)  \ln \left( 1 + \kappa \sum_{i=1}^T \|\g_i\|^2   \right)   \nn\\
&\overset{\step{7}}{\leq}& \ts \ts (\lambda_1+\lambda_2) \cdot (1+\ln \kappa  ) \ln \left( e + \sum_{i=1}^T \|\g_i\|^2   \right)  \nn\\
&\leq & \ts \underbrace{\ts (\lambda_1+\lambda_2) \cdot (1+\ln \kappa  ) \cdot \ln(e+ \hG^2) }_{:=C_b} \cdot \ln (e+T),
\eeq
\noi where step \step{1} uses the definition of $B_t$; step \step{2} uses the results of shifting the summation since $\|\m_0\|^2=0$; step \step{3} uses the definition of $\gamma$, leading to $\gamma_t^2\leq \alpha_t \cdot \tfrac{1}{1+\sum_{i=1}^t \|\m_i\|^2}$; step \step{4} uses the standard sum-to-integral inequality $\sum_{t=1}^T \frac{\x_t}{1 + \sum_{i=1}^t \x_i} \leq \ln(1 + \sum_{t=1}^T \x_t)$ with $\x_i=\|\m_i\|^2$ or $\x_i=\|\g_i\|^2$; step \step{5} uses $\sum_{t=1}^T \|\m_t\|^2\leq \kappa \sum_{t=1}^T \|\g_t\|^2$; step \step{6} uses $\kappa\geq 1$; step \step{7} uses $\ln(1+\kappa a) \leq (1+\ln \kappa ) \cdot \ln(e + a)$ for all $\kappa\geq 1$ and $a\geq 0$. 
 
\paragraph{Bounding the term $\EEE [\sum_{t=1}^T \gamma_t \|\s_t\|^2]$.} Using Lemma \ref{lemma:bound:eee:1}, we derive the following inequalities for all $t\geq 1$:
\beq
0 &\le& \EEE [B_t + \big(1-\tfrac{\alpha_t}{4}\big)\|\s_{t-1}\|^2  - \|\s_t\|^2 ] \nn \\
& =  & \EEE [B_t + (1 - \tfrac{\alpha_t}{4} )\|\s_{t-1}\|^2 - (1-\tfrac{\alpha_{t+1}}{4} )\|\s_t\|^2 - \tfrac{\alpha_{t+1}}{4} \|\s_{t}\|^2   ] \nn \\
&\overset{\step{1}}{\leq}& \EEE [   B_t + (1 - \tfrac{\alpha_t}{4} )\|\s_{t-1}\|^2 - (1-\tfrac{\alpha_{t+1}}{4} )\|\s_t\|^2    - \tfrac{\alpha_{t}}{4 (1 +{\hG})} \|\s_{t}\|^2 ], \nn 
\eeq
\noi step \step{1} uses $\tfrac{\alpha_t}{\alpha_{t+1}}\leq 1 + \hG$.

Multiplying both sided by $4 (1+{\hG})$ and summing the resulting inequality over $t$ from $1$ to $T$ allows the sequence to telescope:
\beq \label{eq:to:log}
\ts \EEE [\sum_{t=1}^T \alpha_t \|\s_t\|^2] &\leq&  \ts  4 (1+{\hG}) \cdot   \sum_{t=1}^T \EEE \big[    (1-\tfrac{\alpha_t}{4}) \|\s_{t-1}\|^2 -  (1-\tfrac{\alpha_{t+1}}{4}) \|\s_{t}\|^2    +   B_t  \big] \nn\\
&= &  \ts  4 (1+{\hG}) \cdot  \EEE \big[  (1-\tfrac{\alpha_0}{4}) \|\s_{0}\|^2    +  \sum_{t=1}^T B_t  \big] \nn\\
&\overset{\step{1}}{\leq}&  \ts  4 (1+{\hG})\cdot {\hG}^2 + 4 (1+{\hG}) C_b\ln(e+T)\nn\\
&\overset{}{\leq}& \ts \underbrace{\ts4 (1+{\hG})\cdot ({\hG}^2 + C_b) }_{:=C_c} \cdot \ln(e+T),\nn
\eeq
\noi where step \step{1} uses $\alpha_0=\rho_0=1$ and $\|\s_0\|=\|\m_0 - \nabla f(\x_0)\|\leq {\hG}$.

\end{proof}

\subsection{Proof of Lemma \ref{lemma:OptEMAM:G:bound}}
\label{app:lemma:OptEMAM:G:bound}

\begin{proof} Define $F_t:=f(\x_t)-f_*$, $\GGG_t:=\sum_{i=1}^t\|\g_i\|^2$, $\MMM_t:=\sum_{i=1}^t\|\m_i\|^2$.

\noi Define $\gamma_t = \min \left(
\alpha_t,\,
\sqrt{\alpha_t}\left( 1 + \sum_{j=1}^t \|\m_j\|^2 \right)^{-1/2}
\right)$.

\paragraph{Bounding the term $D_T:=\sum_{t=1}^T \gamma_t^2\|\m_t\|^2$.} We have almost surely:
\beq \label{eq:DT}
\ts D_T \leq \sum_{t=1}^T\frac{\|\m_t\|^2}{1+\MMM_t}\,\, \xixi{1}{\le}\,\, \ln(1+\MMM_T)
\,\,\xixi{2}{\le}\,\, \ln(1+\kappa \hG^2 T)\,\,\xixi{3}{\le}\,\, \underbrace{\ts \ln(e+\kappa \hG^2 )}_{:=C_d} \ln(e+T),
\eeq
\noi where step \step{1} uses $\gamma_t^2\leq \alpha_t(1+\MMM_t)^{-1}\leq (1+\MMM_t)^{-1}$; step \step{2} uses $\MMM_T\leq \kappa \GGG_T\leq \kappa \hG^2 T$; step \step{3} uses $\ln(1+ab)\leq \ln(e+a)\ln(e+b)$ for all $a,b\geq 0$.

\paragraph{Bounding the term $\EEE[F_T]$.} By Lemma \ref{lemma:OptEMA:M:suff:dec}, after dropping the nonnegative descent term, we have
\beq
F_{t+1}-F_t \leq c_1\gamma_t\|\s_t\|^2+c_2\gamma_t^2\|\m_t\|^2 . \nn
\eeq
Taking expectations and summing from $t=1$ to $T-1$ gives
\beq\label{eq:OptEMA:M:objective:expectation:first}
\EEE[F_T] &\leq& \ts F_1 + c_1 \EEE [\sum_{t=1}^{T-1}\gamma_t\|\s_t\|^2] + c_2\EEE [\sum_{t=1}^{T-1}\gamma_t^2\|\m_t\|^2] \nn\\
& \xixi{1}{\leq} & \ts F_1 + c_1 C_c \ln(e+T) + c_2 C_d \ln(e+T) \nn\\
& \xixi{2}{\leq} & \ts \underbrace{\ts (F_1 + c_1 C_c + c_2 C_d )}_{C_f} \cdot \ln(e+T), \nn
\eeq
\noi where step \step{1} uses Lemma \ref{lemma:sum:tail} that $\EEE\left[\sum_{t=1}^T \alpha_t \|\s_t\|^2\right] \le C_c \ln(e+T)$, and Inequality (\ref{eq:DT}); step \step{2} uses $\ln(e+T)\ge 1$ to absorb the constant
$F_1$.

\end{proof}

\subsection{Proof of Lemma \ref{lemma:M:U2}}
\label{app:lemma:M:U2}

\begin{proof}
Let $U_t:=\sum_{i=1}^t \gamma_i \|\s_i\|^2$, $U_0:=0$.

\paragraph{Bounding the term $\EEE[\sum_{t=1}^T U_{t-1} \gamma_t \|\s_t\|^2]$.} Note that $U_{t-1}\ge0$ is an integrable $\FF_{t-1}$-measurable random variable. Applying Lemma \ref{lemma:bound:eee:1} with $H_{t-1}=U_{t-1}$ yields:
 \beq
 0 &\xixi{}{\le}& \EEE\!\left[U_{t-1}\left(B_t+\big(1-\frac{\alpha_t}{4}\big)\|\s_{t-1}\|^2-\|\s_t\|^2\right)\right] \nn\\
&\xixi{}{=}& \EEE\!\left[U_{t-1}\left(- \frac{\alpha_t}{4}\|\s_t\|^2  + B_t + \big(1-\frac{\alpha_t}{4}\big) (\|\s_{t-1}\|^2 -\|\s_t\|^2 )\right)\right] \nn\\
&\xixi{1}{\leq}& \EEE\!\left[U_{t-1}\left(- \frac{\gamma_t}{4}\|\s_t\|^2  + B_t + \big(1-\frac{\alpha_t}{4}\big) (\|\s_{t-1}\|^2 -\|\s_t\|^2 )\right)\right]. \nn
\eeq
Multiplying the preceding inequality by $4$, and summing over
$t=1,\ldots,T$, we obtain
\beq \label{eq:to:be:com:0}
\ts \EEE \left[\sum_{t=1}^T U_{t-1}\gamma_t\|\s_t\|^2\right] \le
\underbrace{\ts 4\EEE \left[\sum_{t=1}^T U_{t-1}B_t\right]}_{\text{Term I}} + \underbrace{\ts 4\EEE \left[\sum_{t=1}^T
U_{t-1}\left(1-\frac{\alpha_t}{4}\right)
(\|\s_{t-1}\|^2-\|\s_t\|^2)\right]}_{\text{Term II}}.
\eeq

\paragraph{Bounding \text{Term I}.} For \text{Term I}, Lemma \ref{lemma:sum:tail} and the monotonicity of $U_t$ give, almost surely,
\beq
\ts \sum_{t=1}^T U_{t-1}B_t \le U_{T-1}\sum_{t=1}^T B_t \le C_b\ln(e+T)U_{T-1}.\nn
\eeq
Since $\gamma_t\le \alpha_t$, Lemma \ref{lemma:sum:tail} further gives $\EEE[U_{T-1}] \le \EEE[U_T] \le\EEE \left[\sum_{t=1}^T\alpha_t\|\s_t\|^2\right]
\le C_c\ln(e+T)$. This further leads to:
\beq \label{eq:to:be:com:1}
\text{Term I} \leq 4 C_bC_c\ln^2(e+T).
\eeq

\paragraph{Bounding \text{Term II}.} For \text{Term II}, set $A_t:=U_{t-1}\left(1-\frac{\alpha_t}{4}\right)$, $r_t:=\|\s_t\|^2$. Since $U_{t-1}$ is nondecreasing and $\alpha_t$ is nonincreasing, $A_t$ is nondecreasing. Thus,
by Lemma \ref{lemma:sum_by_parts},
\beq
\ts \sum_{t=1}^T A_t(r_{t-1}-r_t) \le A_T\max_{0\le i\le T}r_i \le 4 \hG^2 U_{T-1},\nn
\eeq
where we used Lemma \ref{lemma:bound:EMA:M:v} and $\|\s_0\|=\|\nabla f(\x_1)\|\le \hG$. Therefore,
\beq \label{eq:to:be:com:2}
\text{Term II} \le
16 \hG^2 C_c\ln(e+T).
\eeq
Combining Inequalities (\ref{eq:to:be:com:0}), (\ref{eq:to:be:com:1}), (\ref{eq:to:be:com:2}) and using $\ln(e+T)\le \ln^2(e+T)$ gives
\beq \label{eq:to:be:com:res}
\ts \EEE \left[\sum_{t=1}^T U_{t-1}\gamma_t\|\s_t\|^2\right]
\le
(4C_bC_c+16\hG^2C_c)\ln^2(e+T).
\eeq

\paragraph{Bounding the term $\EEE[U_T^2]$.} We derive:
\beq
\ts U_T^2 \,\,= \,\, \ts \left( \sum_{t=1}^T \gamma_t \|\s_t\|^2 \right)^2 \,\, =\,\, \left( \sum_{t=1}^T \gamma_t^2 \|\s_t\|^4\right) + \left(2\sum_{t=1}^T U_{t-1} \gamma_t \|\s_t\|^2\right). \nn
\eeq
\noi Taking the expectation of both sides yields:
\beq
\ts \EEE[U_T^2] & \leq & \ts  \sum_{t=1}^T \EEE[\gamma_t^2 \|\s_t\|^4 ] + \EEE[2\sum_{t=1}^T U_{t-1} \gamma_t \|\s_t\|^2 ] \nn\\
& \xixi{1}{\leq} & \ts  4 \hG^2 \sum_{t=1}^T \EEE[ \gamma_t \|\s_t\|^2 ] + \EEE[2\sum_{t=1}^T U_{t-1} \gamma_t \|\s_t\|^2 ] \nn\\
& \xixi{2}{\leq} & \ts  4 \hG^2 \cdot C_c \ln(e+T) + 2 \cdot (4 C_b C_c + 16 C_c \hG^2) \ln^2(e+T) \nn\\
& = & \ts  C_u \ln^2(e+T), \nn
\eeq
\noi where step \step{1} uses $\gamma_t\leq 1$ and $\|\s_t\|^2\leq 4 \hG^2$; step \step{2} uses $\EEE\left[\sum_{t=1}^T \gamma_t \|\s_t\|^2\right] \leq \EEE\left[\sum_{t=1}^T \alpha_t \|\s_t\|^2\right] \le C_c \ln(e+T)$ (see Lemma \ref{lemma:sum:tail}), and Inequality (\ref{eq:to:be:com:res}).

\end{proof}

\subsection{Proof of Lemma \ref{lemma:M:weighted:m}}
\label{app:lemma:M:weighted:m}

\begin{proof} Define $c_0:=\tfrac{\theta}{ 2(1+{\hG})}$, $c_1:=\tfrac{\theta}{ 2\varepsilon^2} (1+{\hG})$, and $c_2:=\tfrac{L\theta^2}{2\varepsilon^2}$. 

\noi Let $D_T:=\sum_{t=1}^T\gamma_t^2\|\m_t\|^2$, $U_T:=\sum_{t=1}^T\gamma_t\|\s_t\|^2$, $W_T:=\sum_{t=1}^T\gamma_t\|\m_t\|^2$, and $F_t:=f(\x_t)-f_*$.

\paragraph{Bounding the term $\EEE[W_T]= \EEE[\sum_{t=1}^T\gamma_t\|\m_t\|^2]$.} Summing Inequality (\ref{eq:OptEMA:M:suff:dec:2}) from $t=1$ to $T$ gives the pathwise estimate
\beq
0 &\le& \ts f(\x_1) - f(\x_{T+1}) - \big(c_0 \sum_{t=1}^T\gamma_t\|\m_t\|^2 \big)  + c_1 \big(\sum_{t=1}^T\gamma_t\|\s_t\|^2\big) + c_2 \big(\sum_{t=1}^T\gamma_t^2\|\m_t\|^2\big) \nn\\
&\xixi{1}{\leq}& \ts f(\x_1) - f_* - c_0 W_T + c_1 U_T + c_2 D_T. \label{eq:before:square}
\eeq
\noi where step \step{1} uses $F_{T+1}\ge0$. Dividing both sides by $c_0$ and taking expectations yields
\beq
\EEE[W_T] \,\,\leq\,\,  \tfrac{1}{c_0}\EEE[ F_1 +  c_1 U_T +  c_2 D_T] \,\,\xixi{1}{\leq} \,\, \underbrace{ \tfrac{1}{c_0}\big( \ts F_1 + c_1 C_c + c_2 C_d \big) }_{C_w}  \cdot \ln(e+T) , \nn
\eeq
\noi where step \step{1} uses Lemma \ref{lemma:sum:tail} and Lemma \ref{lemma:OptEMAM:G:bound}.

\paragraph{Bounding the term $\EEE[W_T^2]$.}  Squaring the pathwise estimate in Inequality (\ref{eq:before:square}) gives
\beq
W_T^2 \,\,\leq\,\, \tfrac{3}{c_0^2} \left( F_1^2 + c_1^2 U_T^2 + c_2^2D_T^2\right) \,\,\xixi{1}{\leq}\,\, \tfrac{3}{c_0^2} \left( F_1^2 + c_1^2 U_T^2 + c_2^2 C_d^2 \ln^2(e+T) \right), \nn
\eeq
\noi where step \step{1} uses Lemma \ref{lemma:OptEMAM:G:bound} that $D_T\leq C_d \ln(e+T)$ almost surely. Taking expectations yields:
\beq
\EEE[W_T^2 ] &\leq&  \tfrac{3}{c_0^2} \left(F_1^2+ c_1^2 \EEE[U_T^2] + c_2^2 C_d^2 \ln^2(e+T) \right) \nn\\
&\xixi{1}{\leq} &  \tfrac{3}{c_0^2} \left(F_1^2+ c_1^2 C_u \ln^2(e+T) + c_2^2 C_d^2 \ln^2(e+T) \right) \nn\\
&\xixi{}{\leq} &  \tfrac{3}{c_0^2} \left(F_1^2+ c_1^2 C_u + c_2^2 C_d^2 \right) \cdot \ln^2(e+T), \nn
\eeq
\noi where step \step{1} uses Lemma \ref{lemma:M:U2} that $\EEE[U_T^2]\le C_{u} \ln^2(e+T)$.

\paragraph{Bounding the term \(\mathbb{E}[\alpha_T \MMM_T]\).} Define $q_T:=\alpha_T \MMM_T$. Using the definition of $\gamma_T$, we obtain
\beq
W_T &\xixi{1}{\ge}& \ts  \gamma_T \MMM_T \,\,=\,\, \min\left\{\MMM_T\alpha_T,\MMM_T\frac{\sqrt{\alpha_T}}{\sqrt{1+\MMM_T}}\right\} \,\,\xixi{2}{=}\,\,\ts  \min\left\{q_T, \sqrt{q_T}\sqrt{\frac{\MMM_T}{1+\MMM_T}}\right\}\nn\\
&\ge& \ts \min\left\{q_T,\sqrt{q_T}\sqrt{\frac{q_T}{1+q_T}}\right\} \,\, = \,\, \frac{q_T}{\sqrt{1+q_T}},\nn
\eeq
\noi where step \step{1} uses the fact that the sequence $\gamma_t$ is non-increasing; step \step{2} uses the definition of $\gamma_t$. For any $q\geq 0$, if
\(w\geq q/\sqrt{1+q}\), then $q\leq 2w+2w^2$. This further leads to:
\beq
q_T \leq 2 W_T  + 2 W_T^2.\nn
\eeq
\noi Taking expectations yields:
\beq
\EEE[\alpha_T \MMM_T] &\leq & 2 \EEE[ W_T + W_T^2 ] \nn\\
&\leq & 2C_w \ln(e+T) + 2  \cdot \tfrac{3}{c_0^2} \left(F_1^2+ c_1^2 C_u + c_2^2 C_d^2 \right) \cdot \ln^2(e+T) \nn\\
&\leq & \underbrace{ \ts \big(  2 C_w + \tfrac{6}{c_0^2} (F_1^2+ c_1^2 C_u + c_2^2 C_d^2 ) \big) }_{:=C_v} \cdot \ln^2(e+T). \nn
\eeq

\end{proof}

\subsection{Proof of Lemma \ref{lemma:M:sum:m}}
\label{app:lemma:M:sum:m}

\begin{proof} Let $\SSS_T:=\sum_{t=1}^T\|\s_t\|^2$, $\GGG_T:=\sum_{t=1}^T\|\g_t\|^2$, $\MMM_T:=\sum_{t=1}^T\|\m_t\|^2$, $\ZZ_T:=\EEE[\sqrt{1+\GGG_T}]$.

\paragraph{Bounding the term $\EEE\!\left[ \sqrt{\SSS_T}\right]$.} We derive the following inequalities:
\beq
\ts \EEE\!\left[\sqrt{\sum_{t=1}^T \|\s_t\|^2}\right] &\xixi{1}{\le}& \ts \EEE\!\left[\sqrt{\frac{1}{\alpha_T}}
\sqrt{\sum_{t=1}^T\alpha_t \|\s_t\|^2} \right] \nn\\
&\xixi{2}{\le}& \ts \sqrt{\EEE\!\left[\frac{1}{\alpha_T}\right]} \cdot \sqrt{\EEE\!\left[\sum_{t=1}^T\alpha_t \|\s_t\|^2 \right]} \nn\\
&\xixi{3}{\le}& \ts  \sqrt{C_c \ln(e+T)} \cdot \sqrt{\ZZ_T} , \label{eq:STST}
\eeq
\noi where step \step{1} uses the fact that $\alpha_t$ is non-increasing, leading to $\sum_{t=1}^T \|\s_t\|^2  \le \frac{1}{\alpha_T}\sum_{t=1}^T\alpha_t \|\s_t\|^2$; step \step{2} uses $\EEE[\sqrt{X}\sqrt{Y}] \leq \sqrt{\EEE[X]}\cdot \sqrt{\EEE[Y]}$, which is implied by Cauchy's inequality; step \step{3} uses $\EEE\!\left[\frac{1}{\alpha_T}\right] \le \EEE\!\left[\sqrt{1+\GGG_T}\right] = \ZZ_T$ (which is implied by $\frac{1}{\alpha_T} = \sqrt{
\frac{1+\GGG_T} {1+\frac{\tau}{T}\GGG_T}} \le \sqrt{1+\GGG_T}$), and Lemma \ref{lemma:sum:tail}.

\paragraph{Bounding the term $\EEE[\sqrt{\MMM_T}]$.} We have:
\beq \label{eq:MTMT}
\EEE[\sqrt{\MMM_T}] \,\,\le\,\,\sqrt{\EEE[\alpha_T\MMM_T]} \sqrt{\EEE \big[\frac1{\alpha_T}\big]}
\,\,\xixi{1}{\le}\,\, \sqrt{C_v} \ln(e+T) \cdot \sqrt{\ZZ_T},
\eeq
\noi where step \step{1} uses Lemma \ref{lemma:M:weighted:m}, and the fact that $\EEE[\frac1{\alpha_T}]
=\EEE[\sqrt{\frac{1+\GGG_T}{1+\frac{\tau}{T}\GGG_T}}] \le \EEE[\sqrt{1+\GGG_T}]:=\ZZ_T$.

We have:
\beq
\ts \EEE [\sqrt{\sum_{t=1}^T\|\nabla f(\x_t)\|^2} ] &\xixi{1}{\le}& \EEE[ \sqrt{2}\sqrt{\MMM_T}+\sqrt{2}\sqrt{\SSS_T} ] \nn\\
&\xixi{2}{\le}& \sqrt{2} \sqrt{C_v} \ln(e+T) \sqrt{\ZZ_T} + \sqrt{2} \sqrt{C_c \ln (e+T) } \cdot \sqrt{\ZZ_T} \nn\\
&\leq & \underbrace{\ts \big(\sqrt{2} \sqrt{C_v} + \sqrt{2} \sqrt{C_c } \big)}_{:=C_z} \cdot \ln(e+T) \sqrt{\ZZ_T}, \nn
\eeq
\noi where step \step{1} uses $\nabla f(\x_t)=\m_t-\s_t$; step \step{2} uses Inequalities (\ref{eq:STST}) and (\ref{eq:MTMT}).

\end{proof}

\subsection{Proof of Theorem \ref{the:it:EMA:M}}
\label{app:the:it:EMA:M}

\begin{proof}
Let $\ZZ_T:=\EEE[\sqrt{1+\GGG_T}]$.

\paragraph{Bounding the term $\ZZ_T$ and $\EEE[\sqrt{\sum_{t=1}^T\|\nabla f(\x_t)\|^2}]$.} By Lemma \ref{lemma:M:sum:m},
\beq \label{eq:EGGGG}
\ts \EEE[\sqrt{\sum_{t=1}^T\|\nabla f(\x_t)\|^2}] \le C_{z} \ZZ_T^{1/2} \cdot \ln(e+T).
\eeq
We further derive the following inequalities:
\beq
\ZZ_T &\xixi{1}{\le} & \ts 1+\sqrt{2} \EEE[\sqrt{\sum_{t=1}^T\|\nabla f(\x_t)\|^2}] +\sqrt{2} \EEE\left[\sqrt{\sum_{t=1}^T\|\nabla f(\x_t) - \nabla f(\x_t;\xi_t)\|^2}\right] \nn \\
&\xixi{2}{\le}& \ts 1 + \sqrt{2} C_z \sqrt{\ZZ_T} \ln(e+T) + \sqrt{2}\sigma\sqrt{T} \nn\\
&\xixi{3}{\le}& \ts 2\max(\sqrt{2} C_{z} \sqrt{\ZZ_T} \ln(e+T) , 1 + \sqrt{2}\sigma\sqrt{T}) \nn\\
&\le& \max( (2\sqrt{2} C_{z} \ln(e+T))^2, 2 + 2 \sqrt{2}\sigma\sqrt{T}) \nn\\
&\le&  (2+8 C_{z}^2) \ln^2(e+T) + 2 \sqrt{2}\sigma\sqrt{T}, \nn
\eeq 
\noi where step \step{1} uses $\|\g_t\|^2\le 2\|\nabla f(\x_t)\|^2+2\|\g_t-\nabla f(\x_t)\|^2$; step \step{2} uses Lemma \ref{lemma:M:sum:m}; step \step{3} uses $a+b\leq 2\max(a,b)$ for all $a,b\geq 0$. Putting this inequality back to Inequality (\ref{eq:EGGGG}) yields:
\beq \label{eq:EGGGG:2}
\ts \EEE[\sqrt{\sum_{t=1}^T\|\nabla f(\x_t)\|^2}] &\le& C_{z} \sqrt{(2+8 C_{z}^2) \ln^2(e+T) + 2 \sqrt{2}\sigma\sqrt{T}} \cdot \ln(e+T) \nn\\
&\leq & C_{z} \left( 2 + 3 C_z \right) \cdot \ln^2(e+T) + 2 C_{z} \sigma^{1/2} T^{1/4} \cdot \ln(e+T). \nn
\eeq

\paragraph{Bounding the term $\EEE[\frac1T\sum_{t=1}^T\|\nabla f(\x_t)\|]$.} We derive:
\beq
\ts \EEE[\tfrac{1}{T}\sum_{t=1}^T\|\nabla f(\x_t)\|] &\le& \ts \tfrac{1}{\sqrt{T}} \cdot \EEE[\sqrt{\sum_{t=1}^T\|\nabla f(\x_t)\|^2} ] \nn\\
&\leq & \ts  \tfrac{1}{\sqrt{T}}\cdot C_{z} \left( (2 + 3 C_z)\cdot \ln^2(e+T) + 2 \sigma^{1/2} T^{1/4} \cdot \ln(e+T)\right)  \nn\\
& = & \ts  \OO(T^{-1/2} \ln^2(e+T) + \sqrt{\sigma} T^{-1/4} \ln(e+T)). \nn
\eeq

\noi If $R$ is uniform over $\{1,\ldots,T\}$ and independent of the algorithmic randomness, then
\[
\EEE\|\nabla f(\x_R)\|
=\EEE\left[\frac1T\sum_{t=1}^T\|\nabla f(\x_t)\|\right].\nn
\]

\end{proof}

\section{Proofs for OptEMA-V}
\label{app:sect:it:EMA:V}

\subsection{Proof of Lemma \ref{lemma:bound:m:using:g:2}}
\label{app:lemma:bound:m:using:g:2}

\begin{proof}
By definition, the update rule is $\m_t = (1-\alpha_t) \m_{t-1} + \alpha_t \g_t$. Using the standard initialization $\m_0 = \mathbf{0}$, we apply the convexity of the squared $L_2$ norm to obtain:
\beq
\ts \|\m_t\|_2^2 \,\, = \,\, \ts \|(1 - \alpha_t)\m_{t-1} + \alpha_t \g_t\|_2^2 \,\,\leq\,\,\ts (1 - \alpha_t)\|\m_{t-1}\|_2^2 + \alpha_t \|\g_t\|_2^2.
\eeq
\noi Summing this inequality over $t = 1, \dots, T$ and using $\|\m\|_0^2 = 0$, we obtain:
\beq
0 &\leq&  \ts \big( \sum_{t=1}^T (1 - \alpha_t) \|\m_{t-1}\|^2 \big) + \big(\sum_{t=1}^T \alpha_t \|\g_{t}\|^2\big)  -\big(\sum_{t=1}^T \|\m_t\|^2\big) \nn\\
&= & \ts   - \|\m_{T}\|^2  - \big( \sum_{t=1}^{T-1} \alpha_{t+1} \|\m_{t}\|^2 \big)+ \big( \sum_{t=1}^T \alpha_t \|\g_{t}\|^2 \big) \nn\\
&\overset{\step{1}}{\leq}&  \ts  - \alpha_T \|\m_{T}\|^2  - \big( \alpha_T \sum_{t=1}^{T-1} \|\m_{t}\|^2 \big) + \big(\sum_{t=1}^T \alpha_1 \|\g_{t}\|^2 \big) \nn\\
&\overset{\step{2}}{=}& \ts  - \alpha_T \big( \sum_{t=1}^{T} \|\m_{t}\|^2 \big) + \big( \sum_{t=1}^T \alpha_1 \|\g_{t}\|^2\big), \nn
\eeq
\noi where step \step{1} uses the fact that $1>\alpha_{t+1} \geq \alpha_T $ for all $t < T$. This further leads to $\sum_{t=1}^T \|\m_{t}\|^2 \leq \tfrac{\alpha_1}{\alpha_T} \sum_{t=1}^T\|\g_{t}\|^2$.

\end{proof}
\subsection{Proof of Lemma \ref{lemma:bound:v}}
\label{app:lemma:bound:v}

\begin{proof} 

\noindent Since $\mathbf{v}_{t} = (1-\beta_t) \mathbf{v}_{t-1} + \beta_t \mathbf{g}_t \odot \mathbf{g}_t$, and assuming $\beta_t \in [0, 1]$, using the convexity of $\lVert \cdot \rVert$ and the inequality $\lVert \mathbf{x} \odot \mathbf{x} \rVert \leq \lVert \mathbf{x} \rVert^2$, we have:
    \begin{align*}
        \lVert \mathbf{v}_t \rVert ~\leq~ (1-\beta_t) \lVert \mathbf{v}_{t-1} \rVert + \beta_t \lVert \mathbf{g}_t \odot \mathbf{g}_t \rVert ~\leq ~\max \left( \lVert \mathbf{v}_{t-1} \rVert, \lVert \mathbf{g}_t \rVert^2 \right).
    \end{align*}

\noindent Using this inequality recursively, and assuming $\mathbf{v}_0 = \mathbf{0}$, we obtain:
\begin{align*}
    \ts \lVert \mathbf{v}_t \rVert ~\leq~ \max \left( \lVert \mathbf{v}_0 \rVert, \max_{i=1}^t \lVert \mathbf{g}_i \rVert^2 \right) ~= ~\max_{i=1}^t \lVert \mathbf{g}_i \rVert^2 ~\leq~ \hG^2.
\end{align*}

\paragraph{Bounding the term $\|\m_t\|$.} We derive:
\beq
\|\m_t\| \,\,\leq\,\, (1-\alpha_t) \|\m_{t-1}\| + \alpha_t \|\g_t\| \,\,\leq \,\, \max(\|\m_{t-1}\| , \|\g_t\|) \,\,\leq \,\, \max_{i=0}^t \|\g_i\|\,\,\leq \,\, \hG. \nn
\eeq
\noi Furthermore, we have:
\beq
\tfrac{\gamma_{t-1}^2}{\gamma_t^2}\, \, = \,\, \tfrac{1 + \MMM_t}{1 + \sum_{i=1}^{t-1} \|\m_i\|^2} \,\,\leq\,\, 1 + \|\m_t\|^2 \,\, \leq\,\, 1 + \hG^2. \nn
\eeq
\noi This leads to $\tfrac{\gamma_{t-1}}{\gamma_t} \leq 1+\hG$.

\end{proof}

\subsection{Proof of Lemma \ref{lemma:OptEMA:V:suff:dec}}
\label{app:lemma:OptEMA:V:suff:dec}

\begin{proof}
Let $\FF_t:=\sigma(\xi_1,\ldots,\xi_t)$. For OptEMA-V, we have
$\beta_t=\rho_t$ and
\[
\m_t=(1-\alpha_t)\m_{t-1}+\alpha_t\g_t,
\qquad
\v_t=(1-\rho_t)\v_{t-1}+\rho_t\g_t^2 .
\]
Here $\alpha_t\in(0,1]$ is a non-increasing real sequence and is independent of the current stochastic gradient $\g_t$.

\noi Define
\[
\Theta_t:=
-\left\langle
\nabla f(\x_t),
\frac{\theta\m_t}{\varepsilon+\sqrt{\v_t}}
\right\rangle,
\qquad
\Deltas_t:=
\frac{\theta\m_t}{\varepsilon+\sqrt{\v_{t-1}}}
-
\frac{\theta\m_t}{\varepsilon+\sqrt{\v_t}}.
\]
Let
\[
h_{t-1}:=\varepsilon+\sqrt{\max(\v_{t-1})},
\qquad
\varpi_t:=\frac{\theta\alpha_t}{4h_{t-1}}>0.
\]
We also set $\gamma_0:=1$ and $\Theta_0:=0$.

We rewrite $\Theta_t$ as
\beq \label{eq:ema:V:Theta:varalpha}
\Theta_t
&=&
\underbrace{
\left\langle
\nabla f(\x_t),
\Deltas_t
\right\rangle
}_{:=\Gamma_1}
+
\underbrace{
\left\langle
\nabla f(\x_t),
\frac{-\theta\m_t}{\varepsilon+\sqrt{\v_{t-1}}}
\right\rangle
}_{:=\Gamma_2}.
\eeq

\paragraph{Bounding the term $\Gamma_1$.}
We derive
\beq \label{eq:ema:V:Gamma:1:varalpha}
\Gamma_1
&\leq&
\|\nabla f(\x_t)\|\cdot\|\Deltas_t\|
\nn\\
&\overset{\step{1}}{\leq}&
\varpi_t\|\nabla f(\x_t)\|^2
+
\frac{1}{4\varpi_t}\|\Deltas_t\|^2
\nn\\
&\overset{\step{2}}{\leq}&
\varpi_t\|\nabla f(\x_t)\|^2
+
\frac{\theta^2}{4\varpi_t}
\|\m_t\|^2
\left\|
\frac{1}{\varepsilon+\sqrt{\v_t}}
-
\frac{1}{\varepsilon+\sqrt{\v_{t-1}}}
\right\|_\infty^2
\nn\\
&\overset{\step{3}}{\leq}&
\varpi_t\|\nabla f(\x_t)\|^2
+
\frac{\theta^2}{4\varpi_t\varepsilon^4}
\|\m_t\|^2
\|\v_t-\v_{t-1}\|_\infty
\nn\\
&\overset{\step{4}}{=}&
\varpi_t\|\nabla f(\x_t)\|^2
+
\frac{\theta^2}{4\varpi_t\varepsilon^4}
\|\m_t\|^2
\rho_t
\|\g_t^2-\v_{t-1}\|_\infty
\nn\\
&\overset{\step{5}}{\leq}&
\varpi_t\|\nabla f(\x_t)\|^2
+
\frac{\theta^2}{4\varpi_t\varepsilon^4}
\|\m_t\|^2
\cdot 2\rho_t \hG^2
\nn\\
&\overset{\step{6}}{\leq}&
\frac{\theta\alpha_t}{4h_{t-1}}
\|\nabla f(\x_t)\|^2
+
\frac{2\theta(\varepsilon+\hG) \hG^2}{\alpha_t\varepsilon^4}
\rho_t\|\m_t\|^2 .
\eeq
Here, step \step{1} uses Young's inequality; step \step{2} follows from the element-wise division bound; step \step{3} uses
\[
\left(
\frac{1}{\varepsilon+\sqrt a}
-
\frac{1}{\varepsilon+\sqrt b}
\right)^2
\leq
\frac{|a-b|}{\varepsilon^4},
\qquad a,b\geq 0;
\]
step \step{4} uses $\v_t-\v_{t-1}=\rho_t(\g_t^2-\v_{t-1})$; step \step{5} uses
$\|\g_t^2-\v_{t-1}\|_\infty\leq 2G^2$; and step \step{6} uses the choice of $\varpi_t$ and $\varepsilon\leq h_{t-1}\leq \varepsilon+\hG$.

\paragraph{Bounding the term $\Gamma_2$.}
Let $\P_{t-1}:=
\Diag\left(\frac{1}{\varepsilon+\sqrt{\v_{t-1}}}\right)$. Using the update rule of $\m_t$, we derive
\beq \label{eq:ema:V:Gamma:2:identity:varalpha}
\Gamma_2
&=&
-\theta\left\langle
\nabla f(\x_t),
\P_{t-1}\m_t
\right\rangle
\nn\\
&=&
-\theta(1-\alpha_t)
\left\langle
\nabla f(\x_t),
\P_{t-1}\m_{t-1}
\right\rangle
-
\theta\alpha_t
\left\langle
\nabla f(\x_t),
\P_{t-1}\g_t
\right\rangle
\nn\\
&=&
(1-\alpha_t)\Theta_{t-1}
+
\theta(1-\alpha_t)
\left\langle
\nabla f(\x_{t-1})-\nabla f(\x_t),
\P_{t-1}\m_{t-1}
\right\rangle
\nn\\
&&
-
\theta\alpha_t
\left\langle
\nabla f(\x_t),
\P_{t-1}\g_t
\right\rangle .
\eeq
Multiplying both sides by $\gamma_{t-1}$ and taking expectations, we get
\beq \label{eq:ema:V:Gamma:2:varalpha}
\EEE[\gamma_{t-1}\Gamma_2]
&=&
\EEE[(1-\alpha_t)\gamma_{t-1}\Theta_{t-1}]
\nn\\
&&
+
\EEE\left[
\theta(1-\alpha_t)\gamma_{t-1}
\left\langle
\nabla f(\x_{t-1})-\nabla f(\x_t),
\P_{t-1}\m_{t-1}
\right\rangle
\right]
\nn\\
&&
-
\EEE\left[
\theta\alpha_t\gamma_{t-1}
\left\langle
\nabla f(\x_t),
\P_{t-1}\g_t
\right\rangle
\right].
\eeq

For the second term, by the $L$-smoothness of $f$ and $\x_t-\x_{t-1}
= -\theta\gamma_{t-1} \frac{\m_{t-1}}{\varepsilon+\sqrt{\v_{t-1}}}$, we have
\beq \label{eq:ema:V:smooth:term:varalpha}
&& \ts \theta(1-\alpha_t)\gamma_{t-1}
\left\langle \nabla f(\x_{t-1})-\nabla f(\x_t),
\P_{t-1}\m_{t-1}
\right\rangle
\nn\\
&\leq& \ts \theta(1-\alpha_t)\gamma_{t-1}
\|\nabla f(\x_{t-1})-\nabla f(\x_t)\|\cdot\|\P_{t-1}\m_{t-1}\|
\nn\\
&\leq& \ts \theta(1-\alpha_t)\gamma_{t-1} L\|\x_t-\x_{t-1}\|
\cdot \frac{\|\m_{t-1}\|}{\varepsilon} \nn\\
&\leq& \ts \frac{\theta^2L}{\varepsilon^2}\gamma_{t-1}^2\|\m_{t-1}\|^2.
\eeq

For the third term, since $\alpha_t$ is independent of the current stochastic gradient $\g_t$, and $\gamma_{t-1}$, $\P_{t-1}$, $\x_t$ and $\nabla f(\x_t)$ are independent of $\xi_t$, Assumption \ref{ass:sigma} gives
\beq \label{eq:ema:V:unbiased:varalpha}
&&
\EEE\left[
\theta\alpha_t\gamma_{t-1}
\left\langle
\nabla f(\x_t),
\P_{t-1}\g_t
\right\rangle
\middle|\FF_{t-1}
\right]
\nn\\
&=&
\theta\alpha_t\gamma_{t-1}
\left\langle
\nabla f(\x_t),
\P_{t-1}\nabla f(\x_t)
\right\rangle
\nn\\
&\geq&
\frac{\theta\alpha_t\gamma_{t-1}}{h_{t-1}}
\|\nabla f(\x_t)\|^2 .
\eeq
Combining \eqref{eq:ema:V:Gamma:2:varalpha}, \eqref{eq:ema:V:smooth:term:varalpha}, and \eqref{eq:ema:V:unbiased:varalpha}, we obtain
\beq \label{eq:ema:V:Gamma:2:final:varalpha}
\ts \EEE[\gamma_{t-1}\Gamma_2]
\,\,\leq\,\, \EEE[(1-\alpha_t)\gamma_{t-1}\Theta_{t-1}]
+ \EEE\left[
\frac{\theta^2L}{\varepsilon^2}
\gamma_{t-1}^2\|\m_{t-1}\|^2
\right]  - \EEE\left[
\frac{\theta\alpha_t\gamma_{t-1}}{h_{t-1}}
\|\nabla f(\x_t)\|^2
\right].
\eeq

\paragraph{Bounding the weighted term $\Theta_t$.}
Multiplying \eqref{eq:ema:V:Theta:varalpha} by $\gamma_{t-1}$, taking expectations, and using
\eqref{eq:ema:V:Gamma:1:varalpha} and \eqref{eq:ema:V:Gamma:2:final:varalpha}, we have
\beq \label{eq:ema:V:Theta:weighted:varalpha}
\EEE[\gamma_{t-1}\Theta_t]
&\leq&
\EEE[(1-\alpha_t)\gamma_{t-1}\Theta_{t-1}]
\nn\\
&&
-
\EEE\left[
\frac{3\theta\alpha_t\gamma_{t-1}}{4h_{t-1}}
\|\nabla f(\x_t)\|^2
\right]
\nn\\
&&
+
\EEE\left[
\frac{\theta^2L}{\varepsilon^2}
\gamma_{t-1}^2\|\m_{t-1}\|^2
+
\frac{2\theta(\varepsilon+\hG)\hG^2}{\alpha_t\varepsilon^4}
\gamma_{t-1}\rho_t\|\m_t\|^2
\right]
\nn\\
&\overset{\step{1}}{\leq}&
\EEE[(1-\alpha_t)\gamma_{t-1}\Theta_{t-1}]
+
\EEE[ E_t-D_t ].
\eeq
Here, step \step{1} uses $h_{t-1}\leq \varepsilon+\hG$ and the definitions of $E_t$ and $D_t$.

\paragraph{Unrolling the weighted recursion.}
Define
\[
c_0:=\frac{\theta \hG^2}{\varepsilon}.
\]
Since $\|\nabla f(\x_t)\|\leq \hG$, $\|\m_t\|\leq \hG$, and
$\varepsilon+\sqrt{\v_t}\geq \varepsilon$, we have
\[
|\Theta_t|\leq c_0.
\]
Moreover, $\gamma_t=(1+\sum_{j=1}^t\|\m_j\|^2)^{-1/2}$ is decreasing. Hence, for $t\geq2$,
\beq \label{eq:shift:gamma:theta:varalpha}
\gamma_{t-1}\Theta_{t-1}
&\leq&
\gamma_{t-2}\Theta_{t-1}
+
c_0(\gamma_{t-2}-\gamma_{t-1}).
\eeq
Combining \eqref{eq:ema:V:Theta:weighted:varalpha} and \eqref{eq:shift:gamma:theta:varalpha} yields, for $t\geq2$,
\beq \label{eq:Phi:recursion:varalpha}
\EEE[\gamma_{t-1}\Theta_t]
&\leq&
(1-\alpha_t)\EEE[\gamma_{t-2}\Theta_{t-1}]
+
\EEE[E_t-D_t]
\nn\\
&&
+
(1-\alpha_t)c_0\EEE[\gamma_{t-2}-\gamma_{t-1}].
\eeq
For $t=1$, the same inequality holds without the last residual term because $\Theta_0=0$.

Since $\alpha_t$ is a real sequence, the coefficients $\{1-\alpha_t\}$ can be unrolled directly. Using $\ts \Pi_{j,t}:=\prod_{k=j+1}^{t}(1-\alpha_k)$, with $\Pi_{t,t}=1$, we obtain
\beq \label{eq:LemTheta_weighted:varalpha}
\ts \EEE[\gamma_{t-1}\Theta_t] &\leq& \ts \EEE\left[
\sum_{j=1}^{t} \Pi_{j,t} (E_j-D_j) \right]  \ts + c_0\EEE\left[\sum_{j=2}^{t} \Pi_{j,t}(1-\alpha_j)(\gamma_{j-2}-\gamma_{j-1})\right].
\eeq

\paragraph{Bounding the decrease.}
By the $L$-smoothness of $f$, we have
\beq \label{eq:dec:replace:gamma:varalpha}
\EEE[f(\x_{t+1})-f(\x_t)] &\leq& \ts \EEE\left[
\frac{L}{2}\|\x_{t+1}-\x_t\|^2 + \left\langle
\nabla f(\x_t),\x_{t+1}-\x_t \right\rangle\right] \nn\\
&=& \ts \EEE\left[\frac{L\theta^2}{2}\gamma_t^2
\left\|\frac{\m_t}{\varepsilon+\sqrt{\v_t}}
\right\|^2+\gamma_t\Theta_t \right] \nn\\
&\leq& \ts \EEE\left[\frac{L\theta^2}{2\varepsilon^2}
\gamma_t^2\|\m_t\|^2 + \gamma_{t-1}\Theta_t + c_0(\gamma_{t-1}-\gamma_t) \right],
\eeq
where the last inequality uses $\gamma_t\leq\gamma_{t-1}$ and $|\Theta_t|\leq c_0$.

Combining \eqref{eq:dec:replace:gamma:varalpha} and \eqref{eq:LemTheta_weighted:varalpha}, we get
\beq
\ts \EEE[f(\x_{t+1})-f(\x_t)] &\leq& \ts \EEE\left[\frac{L\theta^2}{2\varepsilon^2}\gamma_t^2\|\m_t\|^2 \right] +
\EEE\left[
\sum_{j=1}^{t}
\Pi_{j,t} (E_j-D_j)
\right]
\nn\\
&& \ts + c_0\EEE[\gamma_{t-1}-\gamma_t] + c_0\EEE\left[
\sum_{j=2}^{t} \Pi_{j,t}(1-\alpha_j)(\gamma_{j-2}-\gamma_{j-1})
\right].
\eeq
This proves the desired one-step descent inequality.
\end{proof}

\subsection{Proof of Lemma \ref{lemma:alpha:b:gamma}}
\label{app:lemma:alpha:b:gamma}

\begin{proof} Let $S$ denote the middle term of the inequality: $S := \sum_{t=1}^T  \sum_{j=1}^{t}  \left( \prod_{k=j+1}^t (1-\alpha_k) \right) A_j$.

\paragraph{Lower Bound.} By standard mathematical convention, when $j=t$, the empty product is defined as $\prod_{k=t+1}^t (1-\alpha_k) = 1$. Since $\alpha_k \in (0,1]$, it holds that $(1-\alpha_k) \geq 0$. Given that $A_j \geq 0$, all terms in the inner summation are strictly non-negative. By retaining only the final term of the inner sum (i.e., where $j=t$), we establish the lower bound:
\beq
\ts S \geq \sum_{t=1}^T \left( \prod_{k=t+1}^t (1-\alpha_k) \right) A_t =  \sum_{t=1}^T A_t.
\eeq

\paragraph{Upper Bound.} To establish the upper bound, we first exchange the order of summation in $S$:
\beq \label{eq:S_swapped}
\ts S = \sum_{j=1}^T A_j \sum_{t=j}^T \prod_{k=j+1}^t (1-\alpha_k).
\eeq
Let us define the inner summation as $C_j := \sum_{t=j}^T \prod_{k=j+1}^t (1-\alpha_k)$. Since $\alpha_k$ is a decreasing sequence, we have $\alpha_k \geq \alpha_T$ for all $k \leq T$, which implies $1-\alpha_k \leq 1-\alpha_T < 1$.

Using these monotonic properties, we can upper bound $C_j$ by replacing the product with a geometric progression:
\beq
\ts C_j \leq  \sum_{t=j}^T \prod_{k=j+1}^t (1-\alpha_T) =  \sum_{t=j}^T (1-\alpha_T)^{t-j}. \nn
\eeq
By applying a change of variables $i = t-j$, we can upper bound this finite sum by an infinite geometric series:
\beq
\ts C_j \leq \sum_{i=0}^{T-j} (1-\alpha_T)^i \leq \sum_{i=0}^{\infty} (1-\alpha_T)^i = \frac{1}{1 - (1-\alpha_T)} = \frac{1}{\alpha_T}. \nn
\eeq
Finally, substituting this upper bound $C_j \leq \frac{1}{\alpha_T}$ back into Equation (\ref{eq:S_swapped}), we obtain:
$$ \ts S = \sum_{j=1}^T A_j C_j \leq \sum_{j=1}^T A_j \left(\frac{1}{\alpha_T}\right) = \frac{1}{\alpha_T} \sum_{t=1}^T  A_t. $$
This completes the proof.
\end{proof}

\subsection{Proof of Lemma \ref{lemma:OptEMAV:G:bound}}
\label{app:lemma:OptEMAV:G:bound}

\begin{proof}
Let $F_t:=f(\x_t)-f_*$, $\GGG_t:=\sum_{i=1}^t\|\g_i\|^2$, $\MMM_t:=\sum_{i=1}^t\|\m_i\|^2$.

\noi Define $E_t := c_1\gamma_{t-1}^2 \|\m_{t-1}\|^2 + c_2 \tfrac{\rho_t}{\alpha_t}
\gamma_{t-1}\|\m_t\|^2$, and $D_t := c_4 \alpha_t \gamma_{t-1} \|\nabla f(\x_t)\|^2$.

\noi For OptEMA-V, $\alpha_t=\alpha$, $\beta_t=\rho_t$, and $\gamma_t=(1+\MMM_t)^{-1/2}$.  

\noi Define $C_e:=\frac{c_1}{\alpha} + (1+\hG)^2 \cdot \frac{c_2}{\alpha^2}$

\paragraph{Bound the term $\sum_{t=1}^T\gamma_t^2\|\m_t\|^2$.} We have:
\beq
\ts \sum_{t=1}^T\gamma_t^2\|\m_t\|^2 &\xixi{1}{=}& \ts \sum_{t=1}^T\frac{\|\m_t\|^2}{1+\MMM_t}
\,\,\leq \,\,\ln(1+\MMM_T)\,\, \xixi{2}{\le}\,\, \ln(e+\GGG_T) , \nn
\eeq
\noi where step \step{1} uses $\gamma_t=(1+\MMM_t)^{-1/2}$; step \step{2} uses $\MMM_t\le \GGG_t$, which is implied by Lemma \ref{lemma:bound:m:using:g:2} since $\alpha_t$ is constant.

\paragraph{Bounding the term $\sum_{t=1}^{T}\sum_{j=1}^{t}(1-\alpha)^{t-j}E_j$.} We have:
\beq
\ts \sum_{t=1}^{T}\sum_{j=1}^{t}(1-\alpha)^{t-j}E_j &\xixi{1}{\le}& \ts  \frac1\alpha\sum_{t=1}^TE_t \nn\\
&\xixi{}{=}& \ts  \big( \frac{c_1}{\alpha} \sum_{t=1}^T  \gamma_{t-1}^2 \|\m_{t-1}\|^2 \big) + \big( \frac{c_2}{\alpha^2}  \sum_{t=1}^T \rho_t \gamma_{t-1}\|\m_t\|^2 \big) \nn\\
&\xixi{2}{\leq}& \ts  \big( \frac{c_1}{\alpha} \sum_{t=1}^T  \gamma_{t}^2 \|\m_{t}\|^2 \big) + \big( \frac{c_2}{\alpha^2} \sum_{t=1}^T \tfrac{\sqrt{1 + \tfrac{\tau}{t} \GGG_t}}{\sqrt{1 + \GGG_t}} \gamma_{t-1}\|\m_t\|^2 \big) \nn\\
&\xixi{3}{\leq}& \ts  \big( \frac{c_1}{\alpha} \sum_{t=1}^T  \tfrac{\|\m_{t}\|^2}{1 + \MMM_t}  \big) + (1 + \hG)^2 \cdot \big( \frac{c_2}{\alpha^2}  \sum_{t=1}^T \tfrac{\gamma_{t}\|\m_t\|^2}{\sqrt{1 + \GGG_t}}  \big) \nn\\
&\xixi{}{\leq}& \ts \underbrace{\ts ( \frac{c_1}{\alpha} + (1+\hG)^2 \cdot \frac{c_2}{\alpha^2})}_{:=C_e}\cdot \ln(e+\GGG_T), \nn
\eeq
\noi where step \step{1} uses the geometric summation induced by the fixed EMA coefficient; step \step{2} uses $\sum_{t=1}^T\gamma_{t-1}^2\|\m_{t-1}\|^2
\le \sum_{t=1}^T\gamma_t^2\|\m_t\|^2$, $\rho_t \le \frac{1+\sqrt{\tau}\hG}{\sqrt{1+\GGG_t}}\le \frac{1+\hG}{\sqrt{1+\GGG_t}}$, and $\gamma_{t-1}\le(1+\hG)\gamma_t$ (see Lemma \ref{lemma:bound:v}).

\paragraph{Bounding the residual terms.} The residual terms in Lemma \ref{lemma:OptEMA:V:suff:dec} satisfy
\beq
&& \ts c_0\sum_{t=1}^{T}(\gamma_{t-1}-\gamma_t)+(1-\alpha)c_0\sum_{t=1}^{T}\sum_{j=2}^{t}(1-\alpha)^{t-j}(\gamma_{j-2}-\gamma_{j-1}) \nn \\
&\le& \ts c_0+\frac{(1-\alpha)c_0}{\alpha} \,\, \le\,\, \frac{c_0}{\alpha},\nn
\eeq
where we used the monotonicity of $\gamma_t$ and $\gamma_0=1$.

\paragraph{Expected objective growth.}
Dropping the non-positive term involving $D_j$ in Lemma \ref{lemma:OptEMA:V:suff:dec}, summing from $t=1$ to $T-1$, and using the preceding bounds gives
\beq
\EEE[F_T] &\le & \ts \EEE[F_1+\frac{c_0}{\alpha} +  \big(c_3 + C_e \big) \cdot \ln(e+\GGG_T) ]\nn\\
&\xixi{1}{\leq}& \underbrace{ \ts ( F_1+\frac{c_0}{\alpha} + \big(c_3 + C_e \big) \cdot \ln(e+\hG^2) )}_{:=C_f} \cdot \ln(e+T), \nn
\eeq
\noi where step \step{1} uses $\ln(e + \GGG_T) \leq \ln (e+ \hG^2T)\leq \ln (e+ \hG^2)\ln(e+T)$, which can be implied by Assumption \ref{ass:BG}.

\end{proof}

\subsection{Proof of Lemma \ref{lemma:bound:sum:nabla}}
\label{app:lemma:bound:sum:nabla}

\begin{proof}
Let $A_T:=\sum_{t=1}^T\|\nabla f(\x_t)\|^2$, and $\MMM_t:=\sum_{j=1}^t\|\m_j\|^2$.

\noi Define $H_T:=\sum_{t=1}^T\gamma_{t-1}\|\nabla f(\x_t)\|^2$ with $\gamma_0:=1$. 

\noi Let $F_t:=f(\x_t)-f_*$, and $C_h:= \frac{1}{c_4\alpha}(C_e + F_1+c_3 + \tfrac{c_0}{\alpha}) \ln (e+\hG^2)$.

\paragraph{Bounding the term $\sum_{t=1}^T \gamma_t^2 \|\m_t\|^2$.} We derive:
\beq \label{eq:gamma2:m2}
\ts \sum_{t=1}^T \gamma_t^2 \|\m_t\|^2 &\leq & \ts \sum_{t=1}^T \tfrac{\|\m_t\|^2}{1 + \MMM_t} \leq \ln(e + \GGG_T).  
\eeq

 \paragraph{Bounding the term $\EEE[H_T]$.} Summing Lemma \ref{lemma:OptEMA:V:suff:dec} over $t=1,\ldots,T$ and moving the negative terms $D_j$ to the left gives
\beq
\ts 0 &\le& \ts  \EEE [ \big(\sum_{t=1}^T\sum_{j=1}^t(1-\alpha)^{t-j} (E_j-D_j)  \big) + F_1 + c_3 \big(\sum_{t=1}^T\gamma_t^2\|\m_t\|^2\big) 
 + \frac{c_0}{\alpha}] \nn\\
&\xixi{1}{\leq} & \ts   \EEE [ -\big(\sum_{t=1}^T\sum_{j=1}^t(1-\alpha)^{t-j} D_j  \big)   + C_e \ln (e+\GGG_T) + F_1 + c_3 \ln(e+\GGG_T) + \frac{c_0}{\alpha} ] \nn\\
&\xixi{2}{\leq} & \ts -  \EEE [ \big(\sum_{t=1}^T D_t  \big) ] + (C_e + F_1+c_3 + \tfrac{c_0}{\alpha}) \ln (e+\hG^2 T)  \nn\\
&\xixi{3}{\leq} & \ts -  \EEE [ c_4 \alpha H_T ] + (C_e + F_1+c_3 + \tfrac{c_0}{\alpha}) \ln (e+\hG^2) \ln (e+T) , \nn
\eeq
\noi where step \step{1} uses the pathwise bound in Lemma \ref{lemma:OptEMAV:G:bound} that $\sum_{t=1}^T\sum_{j=1}^t(1-\alpha)^{t-j} E_j\leq C_e \ln(e+\GGG_T)$ a.s., and Inequality (\ref{eq:gamma2:m2}); step \step{2} uses $\sum_{t=1}^T\sum_{j=1}^t(1-\alpha)^{t-j}D_j
\ge\sum_{t=1}^TD_t$; step \step{3} uses the definition of $H_T$. This further leads to:
\beq
\EEE[H_T] \le \frac{(C_e + F_1+c_3 + \tfrac{c_0}{\alpha}) \ln (e+\hG^2) \ln (e+T)}{c_4\alpha} = C_h \ln(e+T). \label{eq:HT:ln}
\eeq

\paragraph{Bounding the term $\EEE[\sqrt{A_T}]$.} Since $H_T\ge \gamma_TA_T
=\frac{A_T}{\sqrt{1+\MMM_T}}$, we derive: $\sqrt{A_T}\le \sqrt{H_T}\,(1+\MMM_T)^{1/4}$. By Cauchy--Schwarz, we have:
\beq
\EEE[\sqrt{A_T}] &\le & \ts \left(\EEE[H_T]\right)^{1/2}
\left(\EEE[\sqrt{1+\MMM_T}]\right)^{1/2} \nn\\
&\xixi{1}{\leq} & \sqrt{C_h} \ln^{1/2}(e+T) \cdot \left(\EEE[\sqrt{1+\GGG_T}]\right)^{1/2} ,\nn
\eeq
\noi where step \step{1} uses Inequality (\ref{eq:HT:ln}), and $\MMM_T\le\GGG_T$.
   
\end{proof}

\subsection{Proof of Theorem \ref{the:it:EMA:V}}
\label{app:the:it:EMA:V}

\begin{proof} Let $\ZZ_T:=\EEE[\sqrt{1+\GGG_T}]$.

\paragraph{Bounding the term $\ZZ_T$ and $\EEE[\sqrt{\sum_{t=1}^T\|\nabla f(\x_t)\|^2}]$.} By Lemma \ref{lemma:bound:sum:nabla},
\beq \label{eq:EGGGG:V}
\ts \EEE[\sqrt{\sum_{t=1}^T\|\nabla f(\x_t)\|^2}] \le C_{z} \ZZ_T^{1/2} \cdot \ln^{1/2}(e+T).
\eeq
We further derive the following inequalities:
\beq
\ZZ_T &\xixi{1}{\le} & \ts 1+\sqrt{2} \EEE[\sqrt{\sum_{t=1}^T\|\nabla f(\x_t)\|^2}] +\sqrt{2} \EEE\left[\sqrt{\sum_{t=1}^T\|\nabla f(\x_t) - \nabla f(\x_t;\xi_t)\|^2}\right] \nn \\
&\xixi{2}{\le}& \ts 1 + \sqrt{2} C_z \sqrt{\ZZ_T} \ln^{1/2}(e+T) + \sqrt{2}\sigma\sqrt{T} \nn\\
&\xixi{3}{\le}& \ts 2\max(\sqrt{2} C_{z} \sqrt{\ZZ_T} \ln^{1/2}(e+T) , 1 + \sqrt{2}\sigma\sqrt{T}) \nn\\
&\le& \max\left( 8 C_{z}^2 \ln(e+T), 2 + 2 \sqrt{2}\sigma\sqrt{T}\right) \nn\\
&\le&  (2+8 C_{z}^2) \ln(e+T) + 2 \sqrt{2}\sigma\sqrt{T}, \nn
\eeq 
\noi where step \step{1} uses $\|\g_t\|^2\le 2\|\nabla f(\x_t)\|^2+2\|\g_t-\nabla f(\x_t)\|^2$; step \step{2} uses Lemma \ref{lemma:bound:sum:nabla}; step \step{3} uses $a+b\leq 2\max(a,b)$ for all $a,b\geq 0$. Putting this inequality back to Inequality (\ref{eq:EGGGG:V}) yields:
\beq \label{eq:EGGGG:2}
\ts \EEE[\sqrt{\sum_{t=1}^T\|\nabla f(\x_t)\|^2}] &\le& C_{z} \sqrt{(2+8 C_{z}^2) \ln (e+T) + 2 \sqrt{2}\sigma\sqrt{T}} \cdot \ln^{1/2}(e+T) \nn\\
&\leq & C_{z} \left( 2 + 3 C_z \right) \cdot \ln(e+T) + 2 C_{z} \sigma^{1/2} T^{1/4} \cdot \ln^{1/2}(e+T). \nn
\eeq

\paragraph{Bounding the term $\EEE[\frac1T\sum_{t=1}^T\|\nabla f(\x_t)\|]$.} We derive:
\beq
\ts \EEE[\tfrac{1}{T}\sum_{t=1}^T\|\nabla f(\x_t)\|] &\le& \ts \tfrac{1}{\sqrt{T}} \cdot \EEE[\sqrt{\sum_{t=1}^T\|\nabla f(\x_t)\|^2} ] \nn\\
&\leq & \ts  \tfrac{1}{\sqrt{T}}\cdot C_{z} \left( (2 + 3 C_z)\cdot \ln(e+T) + 2 \sigma^{1/2} T^{1/4} \cdot \ln^{1/2}(e+T)\right)  \nn\\
& = & \ts  \OO(T^{-1/2} \ln(e+T) + \sqrt{\sigma} T^{-1/4} \ln^{1/2}(e+T)). \nn
\eeq

\noi If $R$ is uniform over $\{1,\ldots,T\}$ and independent of the algorithmic randomness, then
\[
\EEE\|\nabla f(\x_R)\|
=\EEE\left[\frac1T\sum_{t=1}^T\|\nabla f(\x_t)\|\right].\nn
\]

\end{proof}

\end{document}